\documentclass{article}

\usepackage{microtype}
\usepackage{graphicx}
\usepackage{subfigure}
\usepackage[table]{xcolor} 

\usepackage{pifont} 
\newcommand{\cmark}{\ding{51}}
\newcommand{\xmark}{\ding{55}}
\usepackage{hyperref}

\usepackage[accepted]{icml2025}

\usepackage{amsmath}
\usepackage{amssymb}
\usepackage{mathtools}
\usepackage{amsthm}
\usepackage{multirow}
\usepackage{booktabs} 
\usepackage[capitalize,noabbrev]{cleveref}

\theoremstyle{plain}
\newtheorem{theorem}{Theorem}[section]
\newtheorem{proposition}[theorem]{Proposition}

\theoremstyle{definition}

\theoremstyle{remark}

\usepackage[textsize=tiny]{todonotes}

\usepackage{amsmath} 

\icmltitlerunning{One Ring to Rule Them All: Unifying Group-Based RL via Dynamic Power-Mean Geometry}

\begin{document}

\twocolumn[
\icmltitle{One Ring to Rule Them All: Unifying Group-Based RL via Dynamic Power-Mean Geometry}



\icmlsetsymbol{equal}{*}
\icmlsetsymbol{corresponding author}{\star}

\begin{icmlauthorlist}
\icmlauthor{Weisong Zhao}{iie1,iie2,sangfor}
\icmlauthor{Tong Wang}{sangfor}
\icmlauthor{Zichang Tan}{sangfor}
\icmlauthor{Te Yang}{ia,ia2}
\icmlauthor{Siran Peng}{ia,ia2}
\icmlauthor{Haoyuan Zhang}{ia,ia2}
\icmlauthor{Tianshuo Zhang}{ia,ia2}
\icmlauthor{Haichao Shi}{iie1,iie2}
\icmlauthor{Meng Meng}{sangfor}
\icmlauthor{Yang Yang}{ia,ia2}
\icmlauthor{Xiangyu Zhu}{ia,ia2}
\icmlauthor{Zhen Lei}{ia,ia2,hk}
\icmlauthor{Xiao-Yu Zhang}{iie1,iie2}
\icmlauthor{Xu Zhou}{sangfor}
\end{icmlauthorlist}

\icmlaffiliation{iie1}{Institute of Information Engineering, Chinese Academy of Sciences}
\icmlaffiliation{iie2}{School of Cyber Security, University of Chinese Academy of Sciences}
\icmlaffiliation{sangfor}{Sangfor Technologies}
\icmlaffiliation{ia}{Institute of Automation, Chinese Academy of Sciences}
\icmlaffiliation{ia2}{School of Artificial Intelligence, University of Chinese Academy of Sciences}
\icmlaffiliation{hk}{CAIR, Hong Kong Institute of Science and Innovation, Chinese Academy
of Sciences}

\icmlcorrespondingauthor{Xu Zhou}{zhouxu@sangfor.com.cn}
\icmlcorrespondingauthor{Xiao-Yu Zhang}{zhangxiaoyu@iie.ac.cn}
\icmlcorrespondingauthor{Zhen Lei}{zhen.lei@ia.ac.cn}

\icmlkeywords{Machine Learning, ICML}

\vskip 0.3in
]



\printAffiliationsAndNotice{}  
\begin{abstract}
Group-based reinforcement learning has evolved from the arithmetic mean of GRPO to the geometric mean of GMPO.
While GMPO improves stability by constraining a conservative objective, it shares a fundamental limitation with GRPO: reliance on a \emph{fixed} aggregation geometry that ignores the evolving and heterogeneous nature of each trajectory.
In this work, we unify these approaches under \textit{Power-Mean Policy Optimization (PMPO)}, a generalized framework that parameterizes the aggregation geometry via the power-mean geometry exponent $p$.
Within this framework, GRPO and GMPO are recovered as special cases ($p=1$ and $p \to 0$). Theoretically, we demonstrate that adjusting $p$ modulates the concentration of gradient updates, effectively reweighting tokens based on their advantage contribution. To determine $p$ adaptively, we introduce a Clip-aware Effective Sample Size (ESS) mechanism. Specifically, we
propose a deterministic rule that maps a trajectory’s clipping fraction to a target ESS. Then, we
solve for the specific $p$ to align the trajectory’s induced ESS with this target one. This allows PMPO to dynamically transition between the aggressive arithmetic mean for reliable trajectories and the conservative geometric mean for unstable ones.
Experiments on multiple mathematical reasoning benchmarks demonstrate that PMPO outperforms strong baselines.
\end{abstract}

\section{Introduction}
Reinforcement learning from human or programmatic feedback (RLHF/RLAIF) has recently emerged as a powerful paradigm for eliciting strong mathematical reasoning capabilities from large language models (LLMs) \citep{deepseek,qwq,qwen3}.
A popular class of methods, typified by Group Relative Policy Optimization (GRPO) \citep{grpo}, optimizes token-level log-probabilities under a bandit-style reward framework.
GRPO adopts an arithmetic mean for token-level importance ratios but exhibits inherent pathologies in long chain-of-thought (CoT) reasoning: susceptibility to ``outlier'' tokens that destabilize importance sampling ratios.
Recent advancements, e.g., Geometric-Mean Policy Optimization (GMPO) \citep{gmpo} and GSPO \citep{gspo}, address these issues by replacing the arithmetic mean with a geometric mean.
This change provably dampens the influence of extreme ratios, effectively stabilizing training.
However, we argue that imposing a \emph{fixed aggregation geometry}, i.e., whether arithmetic or geometric, cannot adapt to the dynamic shifts in signal reliability, leading to either unstable or overly conservative convergence.

In this work, we propose \emph{Power-Mean Policy Optimization (PMPO)}, a unified framework that generalizes group-based RL via the power-mean geometry parameterized by an exponent $p$, as shown in Fig. \ref{fig:compare}.
We show that GRPO and GMPO are merely special cases on this geometry: GRPO corresponds to $p=1$ (Arithmetic), while GMPO corresponds to the limit $p \to 0$ (Geometric).
PMPO breaks the rigidity of fixed geometry by allowing $p$ to vary continuously for each rollout.
Theoretically, we demonstrate that varying $p$ is equivalent to adjusting the inverse temperature of the implicit softmax distribution inherent in the gradient of the power mean. Higher $p$ sharpens the distribution, focusing updates on high-advantage tokens (aggressive), while lower $p$ smooths the distribution (conservative).

The core challenge in this unified framework lies in determining the optimal $p$ for each trajectory.
Directly determining $p$ is ill-posed because its impact on the gradient is not scale-invariant, but rather fluctuates with the magnitude of the log-probabilities.
To address this, we introduce \textit{Clip-aware ESS Matching}.
In PPO-style optimization, clipping indicates trust-region saturation, i.e., signal that the local gradient approximation is becoming unreliable.
Intuitively, trajectories with high clipping rates are ``risky'' and require conservative update (geometric), whereas ``safe'' trajectories with minimal clipping permit aggressive updates (arithmetic).
To formalize this, we leverage the \emph{Effective Sample Size (ESS)} of the induced token weights as a scale-invariant measure of update concentration. We define a deterministic mapping from the trajectory's clipping fraction to a target ESS, and then numerically solve for the specific $p$ that enforces this distribution.
This mechanism inversely scales $p$ with the clipping fraction, i.e., retreating to the conservative geometric mean when instability rises, and turning to the arithmetic mean when signal quality is high.

\begin{itemize}
    \vspace{-4mm}
    \item We establish the power-mean geometry as a generalized view of group-based RL, unifying GRPO ($p=1$) and GMPO ($p \to 0$) into a single continuum.
     \vspace{-2mm}
    \item We propose a novel method to determine the aggregation geometry exponent $p$ by linking trust-region clipping to the Effective Sample Size, providing a scale-invariant control over update conservatism.
     \vspace{-2mm}
    \item We derive the gradient of the PMPO, revealing that $p$ acts as an inverse temperature for a softmax attention mechanism over token-level log-probability changes.
     \vspace{-2mm}
    \item Experiments demonstrate that PMPO establishes a new state-of-the-art for grouped-based RL, effectively balancing stability and efficiency.
     \vspace{-2mm}
\end{itemize}

\section{Related Work}
Group Relative Policy Optimization (GRPO)~\cite{grpo} has established a strong baseline by eliminating value models, with subsequent research focusing on simplifying its core framework, improving stability, and addressing efficiency bottlenecks. While approaches like GPG~\cite{gpg} and PRIME~\cite{prime} streamline the optimization loop by removing surrogate losses or critics, others focus on robust sampling and bias correction; for instance, GSPO~\cite{gspo} and DAPO~\cite{dapo} refine clipping and sequence-level optimization to enhance stability, whereas SRPO~\cite{srpo}, RePO~\cite{repo}, OPO~\cite{opo} tackle trajectoryselection, gradient variance, and history resampling. 
To mitigate length bias and improve the accuracy--efficiency trade-off, methods such as Dr.GRPO~\cite{drgrpo}, GRPO-lead~\cite{grpolead}, and GRPO-$\lambda$~\cite{grpolambda} introduce dynamic penalties and length-aware filtering, complemented by efficiency-driven training recipes like GFPO~\cite{gfpo}, DLER~\cite{dler}, CPPO~\cite{cppo}, and S-GRPO~\cite{sgrpo}. Beyond structural adjustments, significant efforts have been directed toward reward shaping and advantage estimation, e.g., EMPO~\cite{empo}, AAPO~\cite{aapo}, BNPO~\cite{bnpo}, and Seed-GRPO~\cite{seedgrpo}, or even removing external rewards entirely as seen in INTUITOR~\cite{intuitor} and RAFT~\cite{raft}. Finally, these algorithmic innovations are bolstered by exploration techniques like the data-centric contributions such as Open-Reasoner-Zero~\cite{open} and Eurus~\cite{eurus}, which provide high-quality curriculum learning and alignment datasets. GMPO \cite{gmpo} addresses the instability of GRPO by replacing the arithmetic mean with a geometric mean. While group-based RL methods have evolved from the arithmetic mean to the geometric stabilization, a critical limitation remains: the reliance on a \emph{fixed} aggregation geometry.
 \vspace{-3mm}
\section{Method}
\begin{figure*}[!t]
\centering 
\includegraphics[width=1.0\linewidth]{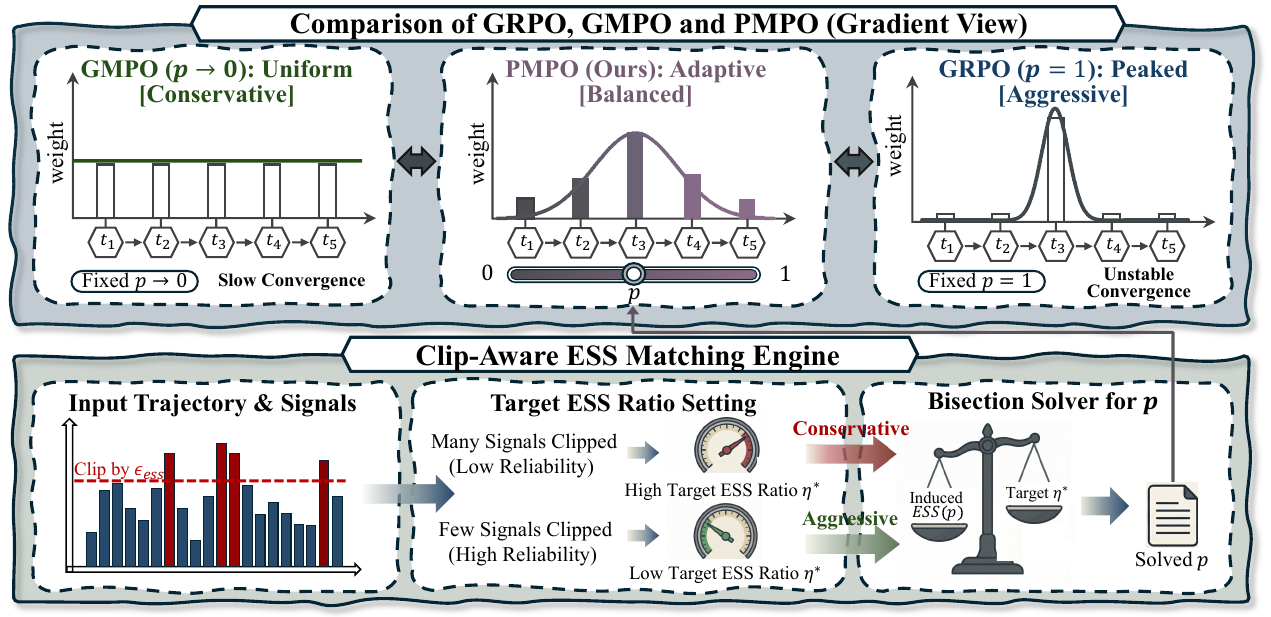}
\vspace{-7mm}
\caption{(a) Gradient view comparison of GRPO, GMPO and PMPO, and (b) Adaptive $p$ via clip-aware ESS matching.}
\vspace{-4mm}
\label{fig:compare}
\end{figure*}
\subsection{Problem Setting and Notation}

We consider online RL training for LLM reasoning with a pre-trained autoregressive model $\pi_\theta$.
For each input problem $x$, the actor samples $K$ full trajectories (reasoning traces) by autoregressive decoding under the old policy $\pi_{\theta_{\text{old}}}$:
\vspace{-1mm}
\begin{equation}
    \tau^{(k)} = \big(s^{(k)}_1, a^{(k)}_1, \dots, s^{(k)}_{T_k}, a^{(k)}_{T_k}\big).
    \vspace{-1mm}
\end{equation}
where $s^{(k)}_t$ represents the state (context history) at step $t$, $a^{(k)}_t$ denotes the action (generated token), and $T_k$ is the length of the trajectory.
A rule-based verifier provides a scalar reward $R^{(k)} \in \{0,1\}$ for each trajectory, indicating whether the final boxed or tagged answer is correct. As in GRPO/GMPO, we treat each trajectory as a bandit episode with a single terminal reward. The log-probability change of token $a^{(k)}_t$ under the new policy relative to the old policy is formulated as:
\vspace{-1mm}
\begin{equation}
    \Delta \ell^{(k)}_t
    \;:=\;
    \log \pi_\theta\big(a^{(k)}_t \mid s^{(k)}_t\big)
    - \log \pi_{\theta_{\text{old}}}\big(a^{(k)}_t \mid s^{(k)}_t\big).
\end{equation}
The per-token importance ratio is defined as:
\vspace{-1mm}
\begin{equation}
    r^{(k)}_t = \exp\big(\Delta \ell^{(k)}_t\big)
    = \frac{\pi_\theta(a^{(k)}_t \mid s^{(k)}_t)}{\pi_{\theta_{\text{old}}}(a^{(k)}_t \mid s^{(k)}_t)}.
    \vspace{-1mm}
\end{equation}
We let $S^{(k)}$ denote the set of valid completion tokens for trajectory $k$, as determined by a response mask (excluding prompt tokens and padding), and write $n^{(k)} = |S^{(k)}|$. We also assume access to a trajectory-level advantage signal $A^{(k)}$.
In the group-based bandit setting, this is computed per-problem by:
\vspace{-3mm}
\begin{equation}
    A^{(k)} = R^{(k)} - \frac{1}{K} \sum_{j=1}^{K} R^{(j)},
    \vspace{-2mm}
\end{equation}
where the sum ranges over the $K$ trajectories associated with the same problem $x$.
More generally, $A^{(k)}$ can be obtained from a value function critic as in PPO.

\subsection{Fixed-Geometry Aggregations}

The core design choice in GRPO- and GMPO-style methods is how to aggregate token-level ratios $\{r^{(k)}_t\}_{t\in S^{(k)}}$ into a single effective ratio $\hat{r}^{(k)}$ that scales the advantage $A^{(k)}$ in the surrogate objective. Broadly speaking, there are two dominant paradigms for this aggregation: the arithmetic mean employed by GRPO \cite{grpo} and the geometric mean introduced by GMPO \cite{gmpo} and GSPO \cite{gspo}, which are detailed in Sec. \ref{sec:grpo} and \ref{sec:gmpo}, respectively. 

\subsubsection{GRPO (arithmetic-mean geometry).}
\label{sec:grpo}
GRPO implicitly uses an \emph{arithmetic mean} of token-level ratios, corresponding to a length-normalized sum:
\vspace{-1mm}
\begin{equation}
\begin{aligned}
    \hat{r}^{(k)}_{\text{GRPO}}
    \;=\;
    \frac{1}{n^{(k)}} \sum_{t\in S^{(k)}} r^{(k)}_t,
    \qquad \\\vspace{-1mm}
    L_{\text{GRPO}}(\theta)
    \;=\;
    - \mathbb{E}_k \big[ A^{(k)} \, \hat{r}^{(k)}_{\text{GRPO}} \big].
\end{aligned}
\vspace{-1mm}
\end{equation}
Equivalently, this can be viewed as a per-token policy gradient loss averaged across tokens.
The arithmetic geometry is sensitive to outliers: a few large ratios can dominate the sum, leading to unstable gradient updates and length bias.

\subsubsection{GMPO (geometric-mean geometry).}
\label{sec:gmpo}
GMPO replaces this arithmetic mean with a geometric mean of ratios:
\vspace{-2mm}
\begin{equation}
\begin{aligned}
    \hat{r}^{(k)}_{\text{GMPO}}
    =
    \Big(\prod_{t\in S^{(k)}} r^{(k)}_t \Big)^{\frac{1}{n^{(k)}}}
    =
    \exp\Big(
        \frac{1}{n^{(k)}} \sum_{t\in S^{(k)}} \Delta \ell^{(k)}_t
    \Big),\\
    \vspace{-1mm}
    with\,\,surrogate:L_{\text{GMPO}}(\theta)
    \;=\;
    - \mathbb{E}_k \big[ A^{(k)} \, \hat{r}^{(k)}_{\text{GMPO}} \big].
\end{aligned}
\end{equation}
In practice, GMPO also performs PPO-style clipping \citep{ppo} in the log-domain:
each $\Delta \ell^{(k)}_t$ is transformed into a clipped version before aggregation, aligned with the sign of the advantage $A^{(k)}$. Similarly, GSPO also adopt geometric mean of tokens but introduces a novel sequence-level clipping mechanism. Overall, this geometric mean dramatically reduces the influence of extreme ratios but fixes the aggregation geometry for all trajectories. Both GRPO and GMPO can therefore be viewed as choosing a \emph{fixed mean} over token-level ratios. Our proposed PMPO generalizes this choice.

\subsection{Power-Mean Policy Optimization}
\subsubsection{Power Mean over Token-Level Ratios}
We propose to aggregate the token-level ratios $\{r^{(k)}_t\}_{t\in S^{(k)}}$ using the \emph{power mean} (generalized mean) of order $p$. We define the \emph{PMPO effective ratio} for trajectory $k$ as:
\vspace{-2mm}
\begin{equation}
\begin{aligned}
    \hat{r}^{(k)}_p
    &=
    (
        \frac{1}{n^{(k)}} \sum_{t\in S^{(k)}} \big(r^{(k)}_t\big)^p
    )^{\!\frac{1}{p}}\\ \vspace{-2mm}
    &=
    (
        \frac{1}{n^{(k)}} \sum_{t\in S^{(k)}} \exp\big(p \, \Delta \ell^{(k)}_t\big)
    )^{\!\frac{1}{p}}.
    \label{eq:pmpo_ratio}
    \end{aligned}
    \vspace{-2mm}
\end{equation}
Given this effective ratio, the PMPO surrogate (before clipping) takes the form:
\begin{equation}
    L_{\text{PMPO}}(\theta)
    =
    - \mathbb{E}_k \big[ A^{(k)} \, \hat{r}^{(k)}_p \big].
    \vspace{-2mm}
\end{equation}

Two important special cases emerge: \\
1) When $p \to 0$, it turns to geometric-mean-like GMPO:
\vspace{-1mm}
    \begin{equation}
        \hat{r}^{(k)}_0
        =
        \exp\Big(
            \frac{1}{n^{(k)}} \sum_{t\in S^{(k)}} \Delta \ell^{(k)}_t
        \Big)
        =
        \hat{r}^{(k)}_{\text{GMPO}}.
        \vspace{-2mm}
    \end{equation}
    2)  When $p=1$, it turns to Arithmetic-mean-like GRPO:
    \vspace{-1mm}
    \begin{equation}
        \hat{r}^{(k)}_1
        =
        \frac{1}{n^{(k)}} \sum_{t\in S^{(k)}} \exp(\Delta \ell^{(k)}_t)
        =
        \hat{r}^{(k)}_{\text{GRPO}},
        \vspace{-1mm}
    \end{equation}
which corresponds to an arithmetic mean over ratios (GRPO). Thus, PMPO provides a continuum of aggregation geometries indexed by $p$, with GMPO and GRPO-like behavior at the two ends of the interval. Mathematically, the power mean $M_p(\{x_t\})$ is well-defined for any real $p\in\mathbb{R}$ (recovering the geometric mean in the limit $p \to 0$) as long as $x_t>0$.
Moreover, for any fixed positive set $\{x_t\}$, the generalized mean is monotone in its order:
if $p_1 < p_2$ then $M_{p_1}(\{x_t\}) \le M_{p_2}(\{x_t\})$, proven in Appendix \ref{app:proof_monotonicity}.

\subsubsection{Gradient Analysis of PMPO}
\label{sec:grad_analysis}
A key property of power means is that the gradient structure is particularly interpretable.
For clarity, we suppress the trajectory index $k$ and let $n = |S|$.
Recall that $\hat{r}_p = (\frac{1}{n} \sum_{t\in S} r_t^p)^{1/p}$.
To analyze how token-level shifts influence the aggregated objective, we first compute the partial derivative of $\hat{r}_p$ w.r.t. an individual ratio $r_j$:
\vspace{-2mm}
\begin{equation}
\begin{aligned}
    (\hat{r}_p)^p
    =
    \frac{1}{n} \sum_{t\in S} r_t^p
      &\Longrightarrow  
    p \hat{r}_p^{p-1} \frac{\partial \hat{r}_p}{\partial r_j}
    =
    \frac{p}{n} r_j^{p-1} \\ \vspace{-4mm}
    &\Longrightarrow 
    \frac{\partial \hat{r}_p}{\partial r_j}
    =
    \frac{\hat{r}_p^{1-p}}{n} r_j^{p-1}.
\end{aligned}
\vspace{-2mm}
\end{equation}
Using the chain rule with $r_j = \exp(\Delta \ell_j)$ (where $\frac{\partial r_j}{\partial \Delta \ell_j} = r_j$), we obtain the gradient w.r.t. the log-probability change:
\vspace{-2mm}
\begin{equation}
\label{gradient}
\begin{aligned}
    \frac{\partial \hat{r}_p}{\partial \Delta \ell_j}
    &= \frac{\partial \hat{r}_p}{\partial r_j} \cdot r_j 
     = \frac{\hat{r}_p^{1-p}}{n} r_j^{p} = \frac{\hat{r}_p}{n} 
    \frac{r_j^p}{
        \frac{1}{n} \sum_{t\in S} r_t^p
    } 
    \\
    &= \frac{\hat{r}_p}{n}
    \cdot
    \underbrace{
        \frac{\exp(p \Delta \ell_j)}{\sum_{t\in S} \exp(p \Delta \ell_t)}
    }_{\text{softmax}_p(\Delta \ell)_j}.
\end{aligned}
\vspace{-1mm}
\end{equation}
Therefore, the gradient of $\hat{r}_p$ with respect to each token-level log-probability change $\Delta \ell_j$ is proportional to a \emph{softmax over $\Delta \ell_t$ with temperature $\frac{1}{p}$}.
The induced policy gradient contribution from token $j$ in a single trajectory is:
\vspace{-1mm}
\begin{equation}
    \frac{\partial L_{\text{PMPO}}}{\partial \Delta \ell_j}
    \propto
    - A \, \frac{\partial \hat{r}_p}{\partial \Delta \ell_j}
    \propto
    - A \, \hat{r}_p \, \text{softmax}_p(\Delta \ell)_j.
    \vspace{-0.5mm}
\end{equation}
This leads to the following interpretation: 1) For \textbf{small $p$} (close to $0$), the softmax distribution over $\Delta \ell_t$ is nearly \emph{uniform}, so all tokens in the completion share gradient mass more equally, mirroring the robustness of the geometric mean. 2) For \textbf{larger $p$} (close to $1$), the softmax becomes more \emph{peaked} around tokens with large positive $\Delta \ell_t$, meaning that PMPO concentrates the update on tokens whose probability under the new policy increased the most, effectively emphasizing decisive reasoning steps.
By choosing $p$ adaptively per rollout, PMPO controls this softmax temperature based on the local signal structure of the trajectory, yielding a principled mechanism for adaptive token reweighting in group-based RL for LLMs. Fig. \ref{fig:compare}(a) visually illustrates this gradient behavior.

\subsubsection{Log-Domain Clipping and PMPO Variants}

To retain the numerical stability and conservative update behavior of PPO-like methods, PMPO follows GMPO in applying clipping in the log-domain on a per-token basis before aggregating via the power mean. Concretely, we first compute unclipped log-probability differences as follows:
\begin{equation}
    \Delta \ell_t
    =
    \log \pi_\theta(a_t \mid s_t)
    - \log \pi_{\theta_{\text{old}}}(a_t \mid s_t),
    \vspace{-1mm}
\end{equation}
then define a sign-corrected and clipped version:
\vspace{-2mm}
\begin{equation}
    \tilde{\Delta} \ell_t
    =
    \operatorname{sgn}(A)
    \cdot
    \max\!\big(\!
        \operatorname{sgn}(A) \, \Delta \ell_t,
        \operatorname{clip}(\operatorname{sgn}(A) \Delta \ell_t, -c, c)
    \big),
\end{equation}
where $c > 0$ is a clipping threshold corresponding to PPO's clipping range in log-space.
PMPO then computes its effective ratio using these clipped log-differences:
\vspace{-1mm}
\begin{equation}
    \hat{r}_p
    =
    M_p\big(\{\exp(\tilde{\Delta} \ell_t)\}_{t\in S}\big).
    \vspace{-2mm}
\end{equation}
Our implementation supports a family of ablation variants: \textbf{PMPO (default)} uses log-domain clipping and length-normalized aggregation (power mean with $1/n$ normalization). \textbf{PMPO-noclip} sets a very large clip range, effectively disabling clipping but retaining the power-mean aggregation. \textbf{PMPO-seqclip} applies clipping to the sequence log-difference $\sum_{t\in S} \Delta \ell_t$ rather than per-token \citep{gspo}. \textbf{PMPO-without-norm} removes the $1/n$ normalization \citep{drgrpo}, using a ``sum-aggregated'' power mean to study length bias and scale sensitivity.

\subsubsection{Determining $p$ via ESS Matching}
\label{sec:ess}
For stability in ESS calculation...
The central design choice in PMPO is how to determine the exponent $p$ for each rollout.
In this paper, we use a simple and deterministic rule: ESS matching of the induced token weights, as illustrated in Fig. \ref{fig:compare}. For stability in ESS calculation, we use the clipped values $\tilde{\Delta}\ell_t$ to compute the induced weights.
We constrain $p \in [p_{\min}, p_{\max}]$ ($[0.01, 0.99]$) to prevent numerical instability at the boundaries. Our gradient analysis shows that PMPO implicitly induces token weights:
\vspace{-2mm}
\begin{equation}
    w_t(p) = \mathrm{softmax}\big(p \, \tilde{\Delta}\ell_t\big).
    \vspace{-2mm}
\end{equation}
We select $p$ by \emph{solving} a monotone matching condition on $w(p)$ via the normalized effective sample size (ESS) ratio:
\vspace{-2mm}
\begin{equation}
    \mathrm{ESS}_{\mathrm{norm}}(p)
    =
    \frac{1}{n}\cdot\frac{1}{\sum_{t\in S} w_t(p)^2}
    \in \big[\tfrac{1}{n}, 1\big].
    \vspace{-2mm}
\end{equation}
ESS is a standard concentration measure for a probability vector \citep{ess}.
It admits an immediate interpretation as the \emph{effective number of tokens} receiving weight:
if the weight mass is distributed uniformly over exactly $k$ tokens (and zero on the remaining $n-k$), i.e.,
$w_t=\frac{1}{k}$ for $t$ in some subset of size $k$ and $w_t=0$ otherwise, then:
\vspace{-3mm}
\begin{equation}
\begin{aligned}
    \mathrm{ESS}(w)=&\frac{1}{\sum_{t=1}^n w_t^2}
    =
    \frac{1}{k\cdot(1/k)^2}
    = k,\\
    \vspace{-2mm}
    &\mathrm{ESS}_{\mathrm{norm}}(w)=\frac{k}{n}.
\end{aligned}
\vspace{-2mm}
\end{equation}
Thus, an ESS target $\eta^\star$ directly specifies the desired \emph{effective token fraction} under the induced softmax weighting. GMPO corresponds to $p\to 0$, which makes $w_t$ nearly uniform and yields $\mathrm{ESS}_{\mathrm{norm}}\approx 1$.
In contrast, GRPO corresponds to the aggregation order $p=1$, but its induced weights $w_t\propto \exp(\tilde{\Delta}\ell_t)$ generally have a \emph{rollout-dependent} concentration; therefore, GRPO does \emph{not} correspond to a single, well-defined ESS.
Rather than forcing a fixed target to ``match GRPO'', we treat ESS as a controllable concentration knob and set the target using a stability signal.

\begin{algorithm}[!t]
    \caption{Power-Mean Policy Optimization (PMPO)}
    \label{alg:pmpo}
    \begin{algorithmic}[1]
        \REQUIRE Policies $\pi_{\theta}, \pi_{\theta_{\text{old}}}$, group size $K$, clip range $c$, ESS threshold $\epsilon_{\mathrm{ess}}$
        
        \STATE Sample trajectories $\{\tau^{(k)}\}_{k=1}^K$ via $\pi_{\theta_{\text{old}}}$ and compute advantages $\{A^{(k)}\}$
        
        \FOR{each trajectory $k = 1 \dots K$}
            \STATE \textbf{\textit{\# Phase I: Signal Reliability Analysis}}
            \STATE Compute per-token log-probability differences $\Delta\ell^{(k)}_t$ and clipped values $\tilde{\Delta}\ell^{(k)}_t$ using $c$
            \STATE Calculate clip fraction $f_{\mathrm{clip}}^{(k)}$ based on threshold $\epsilon_{\mathrm{ess}}$
            
            \STATE \textbf{\textit{\# Phase II: Geometry Adaptation}}
            \STATE Let $n = |S^{(k)}|$ and set $\eta^{\star} \leftarrow \frac{1}{n} + f_{\mathrm{clip}}^{(k)}(1-\frac{1}{n})$
            \STATE Solve for $p^{(k)}$ s.t. $\mathrm{ESS}_{\mathrm{norm}}(p^{(k)}) = \eta^{\star}$
            
            \STATE \textbf{\textit{\# Phase III: Aggregation}}
            \STATE $\hat{r}^{(k)} \leftarrow \left( \frac{1}{n} \sum_{t} \exp(p^{(k)} \tilde{\Delta}\ell^{(k)}_t) \right)^{1/p^{(k)}}$
        \ENDFOR
        \STATE \textbf{return} Gradient update $\theta \leftarrow \theta + \alpha \nabla \frac{1}{K}\sum_{k} A^{(k)} \hat{r}^{(k)}$
    \end{algorithmic}
\end{algorithm}

\paragraph{Clip-aware target ESS via Signal Reliability.}
We determine the target ESS $\eta^\star$ by interpreting the clipping mechanism as a proxy for signal reliability.
Tokens that trigger clipping indicate trust region saturation, i.e., points where the importance ratio deviates significantly, rendering the local gradient approximation unreliable.
Conceptually, we view unclipped tokens as carrying valid, high-fidelity learning signals, while clipped tokens represent saturated or noisy signals.
When the clip fraction is high, the trajectory's overall reliability decreases, necessitating a conservative fallback to the robust geometric mean (high ESS).
Conversely, a low clip fraction signals reliable gradients, permitting a more aggressive, arithmetic-like update (low ESS).

To operationalize this, we borrow the clipping operator with a threshold $\epsilon_{\mathrm{ess}}$ (different from the training clip range $c$) and compute the fraction of saturated response tokens:
\vspace{-2mm}
\begin{equation}
    f_{\mathrm{clip}}
    :=
    \frac{1}{n}\sum_{t\in S} \mathbf{1}\!\left[
        \Delta \ell_t \neq
        \operatorname{clip}(\Delta \ell_t,-\epsilon_{\mathrm{ess}},\epsilon_{\mathrm{ess}})
    \right].
    \vspace{-2mm}
\end{equation}
We then set the target normalized ESS $\eta^\star$ via a linear interpolation based on this saturation rate:
\vspace{-1mm}
\begin{equation}
    \eta^\star
    =
    \tfrac{1}{n}
    +
    f_{\mathrm{clip}}\Big(1-\tfrac{1}{n}\Big)
    \in \big[\tfrac{1}{n}, 1\big],
    \vspace{-1mm}
\end{equation}
where the factor $(1-\tfrac{1}{n})$ ensures the correct range: $\eta^\star=\tfrac{1}{n}$ when no tokens are clipped ($f_{\mathrm{clip}}=0$), and $\eta^\star=1$ when all response tokens are clipped ($f_{\mathrm{clip}}=1$), i.e., a simple length-normalized linear mapping from the clip fraction to the ESS. Note that while $\eta^\star=1/n$ theoretically implies max-pooling ($p \to \infty$), in practice, this unreachable target forces the solver to hit the upper bound $p_{\max}$ (GRPO-like behavior) whenever the signal is reliable.
Intuitively, more clipping indicates that more tokens are pushing against the trust region, so we set a \emph{larger} ESS target (more uniform token weights) to avoid concentrating the update on potentially saturated tokens. Conversely, little or no clipping permits a smaller ESS target (more concentrated weights).
Finally, for each trajectory, we solve for $p \in [p_{\min}, p_{\max}]$ to satisfy $\mathrm{ESS}_{\mathrm{norm}}(p)=\eta^\star$ using bisection. We provide the full pseudo-code of the PMPO in Alg. \ref{alg:pmpo}.

\section{Experiment}
\label{sec:experiment}
\begin{table*}[!t]
    \centering
    \caption{\textbf{Ablation on Component Design.} We analyze the contribution of aggregation geometry, log-domain clipping, and the proposed adaptive power-mean mechanism. Metrics are reported on Qwen2.5-Math-7B.}
    \label{tab:ablation_components}
    \setlength{\tabcolsep}{5pt}
    \begin{tabular}{lcccccccccc}
        \toprule
        \multirow{2}{*}{\textbf{Method}} & \multicolumn{3}{c}{\textbf{Method Configuration}} & \multicolumn{6}{c}{\textbf{Metrics}} \\
        \cmidrule(lr){2-4} \cmidrule(lr){5-10}
         & \textbf{Geometry} & \textbf{Clip} & \textbf{Adaptive $p$} & \textbf{AIME24} & \textbf{AMC} &  \textbf{MATH500} & \textbf{Minerva} & \textbf{Oly.} & \textbf{Avg.} \\
        \midrule
        \rowcolor{gray!10} \textbf{Baselines} & & & & & & & & & \\
        GRPO & Arith. ($p{=}1$) & \cmark  & \xmark & 40.0 & 59.0 & 83.4 & 32.4 & 41.3 & 51.2 \\
        GMPO w/o Clip & Geom. ($p{\to}0$) & \xmark & \xmark & 40.0 & 63.9 & 80.6 & 33.5 & 43.7 & 52.3 \\
        GMPO & Geom. ($p{\to}0$) & \cmark & \xmark & 43.3 & 61.4 & 82.0 & 33.5 & 43.6 & 52.7 \\
        \midrule
        \rowcolor{gray!10} \textbf{PMPO Variants} & & & & & & & & & \\
        PMPO-Fixed & Power($p{=}0.5$) & \cmark & \xmark & 36.7 & 62.7 & 83.8 & 35.7 & 46.1 & 53.0 \\
        PMPO w/o Clip & Power & \xmark & \cmark & 30.0 & 59.0 & 82.8 & 34.2 & 46.0 & 50.4 \\
        \textbf{PMPO} & Power & \cmark & \cmark & 36.7 & 68.7 & 83.8 & 34.9 & 46.7 & \textbf{54.2} \\
        \toprule
    \end{tabular}
    \vspace{-3mm}
\end{table*}
\subsection{Implementation Details}

\noindent\textbf{Models.}
We conduct language-only mathematical reasoning experiments on three base models:
Qwen2.5-Math-1.5B~\citep{qwen2.5}, Qwen2.5-Math-7B\cite{qwen2.5} and DeepSeek-R1-Distill-Qwen-7B\cite{deepseek}.
All methods start from the same pre-trained checkpoints and perform online RL updates on top of the base model.

\noindent\textbf{Training}
Following prior group-based RL setups~\citep{drgrpo}, we train on MATH~\citep{math} Levels 3--5 (8{,}523 problems).
Rewards are verifiable: we assign ``$1$'' to correct responses and ``$0$'' otherwise using a rule-based math verifier. For each prompt, we generate $K=8$ trajectories and cap the maximum response length at 3{,}000 tokens.
In each training round, the old policy $\pi_{\theta_{\mathrm{old}}}$ produces 1{,}024 trajectories, and the current policy $\pi_{\theta}$ is updated for 8 optimization steps with a batch size of 128, resulting in 8 policy updates per trajectorybatch.
We use AdamW with betas $(0.9, 0.95)$ and weight decay $0.0$. Unless specified otherwise, we use a constant learning rate of $10^{-6}$. Detailed hyper-parameters are listed in Tab. \ref{tab:oat_hparams} of Sec. \ref{supp:hyper}. All model training are performed on 8 NVIDIA A800 (80GB) GPUs. 

\noindent\textbf{Evaluation benchmarks and metrics.}
We evaluate on a suite of mathematical reasoning benchmarks spanning different difficulty levels:
AIME24\citep{aime24}, AMC \citep{aime24}, MATH500 \citep{math}, Minerva~\citep{minerva}, and OlympiadBench (Oly.)~\citep{oly}.
We primarily report Pass@1 (accuracy) following Dr.GRPO~\citep{drgrpo}.

\subsection{Ablation Studies}
\label{sec:ablation}

We conduct component-wise ablations to validate the design of PMPO, with results summarized in Tab.~\ref{tab:ablation_components}.
First, to confirm the necessity of dynamic adaptation, we compare PMPO against an optimal fixed geometry ($p=0.5$). While this static compromise outperforms the GMPO, it still lags behind our adaptive approach, verifying that the optimal aggregation geometry is trajectory-dependent rather than universal.
Crucially, we identify log-domain clipping as the fundamental safety rail for stability. Removing this mechanism (PMPO w/o Clip) leads to catastrophic training collapse: although the model initially reaches a peak accuracy of 50.4\%, it rapidly diverges as the training step accumulates, which is due to unbounded arithmetic-like updates violating the trust region.
We also provide analysis for time complexity of PMPO in Appendix. \ref{app:complexity}. In addition, to rigorously validate the necessity of this ESS matching mechanism against a simpler direct mapping (from clipping rate to $p$), we provide a detailed theoretical and empirical comparison in Appendix~\ref{app:scale_invariance}.
\begin{table*}[t]
    \centering
    \caption{Comparison with the existing state-of-the-art methods on five mathematical reasoning benchmarks (AIME24, AMC, MATH500, Minerva and OlympiadBench) for Qwen2.5-Math-1.5B, Qwen2.5-Math-7B, DeepSeek-R1-Distill-Qwen-7B. }
    \label{tab:main_results_math}
    \begin{tabular}{c|ccccc|c}
        \toprule
        Method & AIME24 & AMC & MATH500 & Minerva & Oly. & \textbf{Avg.} \\
        \midrule
        Qwen2.5-Math-1.5B \citep{qwen2.5}  & 16.7 & 43.4 & 61.8 & 15.1 & 28.4 & 33.1\\
        Qwen2.5-Math-1.5B-Instruct \citep{qwen2.5} & 10.0 & 48.2 & 74.2 & 26.5 & 40.2 & 39.8 \\
        Dr.GRPO-1.5B \citep{drgrpo} & 20.0 & 53.0 & 74.2 & 25.7 & 37.6 & 42.1  \\
        GMPO-1.5B \citep{gmpo} & 20.0 & 53.0 & 77.6 & 30.1 & 38.7 & \underline{43.9}\\
        PMPO-1.5B (Ours) & 23.3 & 52.4 & 77.8 & 30.7 & 39.3 & \textbf{44.7} \\
        \midrule
        Qwen2.5-Math-7B \citep{qwen2.5} & 16.7 & 38.6 & 50.6 & 9.9 & 16.6 & 26.5 \\
        SimpleRL-Zero-7B \citep{simplerl} & 26.7 & 60.2 & 78.2 & 27.6 & 40.3 & 46.6 \\
        PRIME-Zero-7B \citep{prime} & 16.7 & 62.7 & 83.8 & 36.0 & 40.9 & 48.0 \\
        OpenReasoner-Zero-7B@3k \citep{open} & 13.3 & 47.0 & 79.2 & 31.6 & 44.0 & 43.0\\
        OpenReasoner-Zero-7B@8k \citep{open} & 13.3 & 54.2 & 82.4 & 31.6 & 47.9 & 45.9  \\
        Eurus-7B \citep{eurus} & 16.7 & 62.7 & 83.8 & 36.0 & 40.9 & 48.0  \\
        GPG-7B \cite{gpg} & 33.3 & 65.0 & 80.0 & 34.2 & 42.4 & 51.0 \\
        Dr.GRPO-7B \cite{drgrpo} & 43.3 & 62.7 & 80.0 & 30.1 & 41.0 & 51.4  \\
        GMPO-7B \cite{gmpo} & 43.3 & 61.4 & 82.0 & 33.5 & 43.6 & \underline{52.7}\\
        PMPO-7B (Ours) & 36.7 & 68.7 & 83.8 & 34.9 & 46.7 & \textbf{54.2} \\
        \midrule
        Dr.GRPO-7B \cite{drgrpo} [R1-Distill] & 50.0 & 74.7 & 89.6 & 37.5 & 55.7 & 61.5\\
        GMPO-7B \cite{gmpo} [R1-Distill] & 46.6 & 78.3 & 91.4 & 37.9 & 62.5 & \underline{63.4}\\
        PMPO-7B (Ours) [R1-Distill] & 46.7 & 79.5 & 93.4 & 39.3 & 64.2 & \textbf{64.6} \\
        \bottomrule
    \end{tabular}
    \vspace{-5mm}
\end{table*}

\subsection{Comparison with State-of-the-Art Methods}
We present the main evaluation results in Tab. ~\ref{tab:main_results_math}. Across all model scales and benchmarks, PMPO consistently outperforms the existing group-based RL baselines.

\noindent\textbf{1.5B Scale.}
On Qwen2.5-Math-1.5B, PMPO achieves an average accuracy of 44.7\%, surpassing the strong geometric-mean baseline GMPO and the arithmetic-mean baseline Dr.GRPO. Notably, PMPO demonstrates robust improvements on hard reasoning tasks, such as AIME24 (+3.3\% vs. GMPO) and OlympiadBench (+0.6\% vs. GMPO).

\noindent\textbf{7B Scale.}
The performance gap widens at the 7B scale. PMPO-7B achieves a state-of-the-art average accuracy of \textbf{54.2\%} among methods trained with comparable budgets.
It significantly outperforms Dr.GRPO (+2.8\%) and GMPO (+1.5\%).
Compared to other recent methods such as PRIME-Zero (48.0\%) and GPG (51.0\%), PMPO establishes a new upper bound for online RL without value function critics. We further validate the robustness of PMPO on the high-performance DeepSeek-R1-Distill-Qwen-7B. Even on this strong reasoning baseline, PMPO yields consistent gains, achieving an average accuracy of 64.6\%. This represents a substantial improvement over Dr.GRPO (61.5\%) and GMPO (63.4\%). Notably, on the MATH500 benchmark, PMPO pushes the accuracy to a remarkable 93.4\%, significantly surpassing Dr.GRPO (89.6\%) and GMPO (91.4\%). 

\subsection{Analysis and Discussion}

\paragraph{Adaptive Aggregation Geometry.}
We track the evolution of the exponent $p$ during training and observe two distinct patterns: the mean value of $p$ stabilizes around $0.8$, while the maximum value $p_{\max}$ consistently saturates at the upper bound throughout the entire process, as shown in Fig. \ref{fig:p}. Mathematical reasoning is inherently signal-sparse: often only a few key steps or the final answer tokens carry the decisive logic. A larger $p$ allows the gradient to concentrate on these high-signal tokens. The equilibrium at $p \approx 0.8$ indicates that the model actively seeks an Arithmetic-like geometry to maximize credit assignment efficiency, only retreating to lower $p$ values when necessary to dampen noise.
In addition, the ceiling phenomenon of $p_{\max}$ forming a straight line at the boundary corresponds to trajectories with zero clipping ($f_{\mathrm{clip}} = 0$). For these ``safe'' trajectories, which lie well within the trust region, our ESS-matching correctly identifies that no conservative damping is needed. Consequently, the algorithm pushes $p$ to its upper limit to approximate the unbiased arithmetic mean as closely as possible. 
To rigorously verify whether the arithmetic mean ($p_{\max}=0.99$) represents the optimal upper limit, we also relax the upper constraint to allow $p > 1$ (e.g., up to $p_{\max}=2$). However, as detailed in Appendix \ref{app:unbounded_p}, we observe that exceeding the arithmetic mean leads to inferior performance due to variance amplification. 

\paragraph{Gradient Stability vs. Efficiency.}
We further investigate the stability of the optimization by monitoring the variance of the policy gradients.
GRPO-style baselines typically exhibit high-frequency spikes in gradient magnitude, indicative of sensitivity to outlier tokens.
Conversely, GMPO maintains extremely low variance but at the cost of slower convergence (as reflected in the performance gap in Table~\ref{tab:main_results_math}).
PMPO strikes a favorable balance: it suppresses the extreme spikes characteristic of GRPO while maintaining a higher signal-to-noise ratio than GMPO, supporting our theoretical claim of adaptive credit assignment.

\paragraph{Sensitivity to ESS Clipping Threshold ($\epsilon_{\mathrm{ess}}$).}
We investigate the impact of the clipping threshold $\epsilon_{\mathrm{ess}}$, which governs the sensitivity of our reliability estimation.
As shown in Fig. \ref{fig:p} (right), setting $\epsilon_{\mathrm{ess}}=0$ forces the mechanism to treat even infinitesimal deviations as "risk," effectively keeping $\eta^\star \approx 1$ and degenerating the method to the conservative GMPO baseline (52.7\%).
Conversely, relaxing the threshold to $\epsilon_{\mathrm{ess}} \ge 0.3$ causes the model to underestimate signal instability, allowing $p$ to remain high even for noisy trajectories. This aggressive behavior leads to performance degradation (50.8\%), approaching the instability of GRPO.
Performance peaks at $\epsilon_{\mathrm{ess}}=0.1$ (54.2\%), which aligns with the standard PPO trust region.
This suggests that the clipping statistic is a highly effective proxy for signal quality, provided the threshold is calibrated to the trust region boundary.
\begin{figure}[t]
\centering
\includegraphics[width=0.49\linewidth]{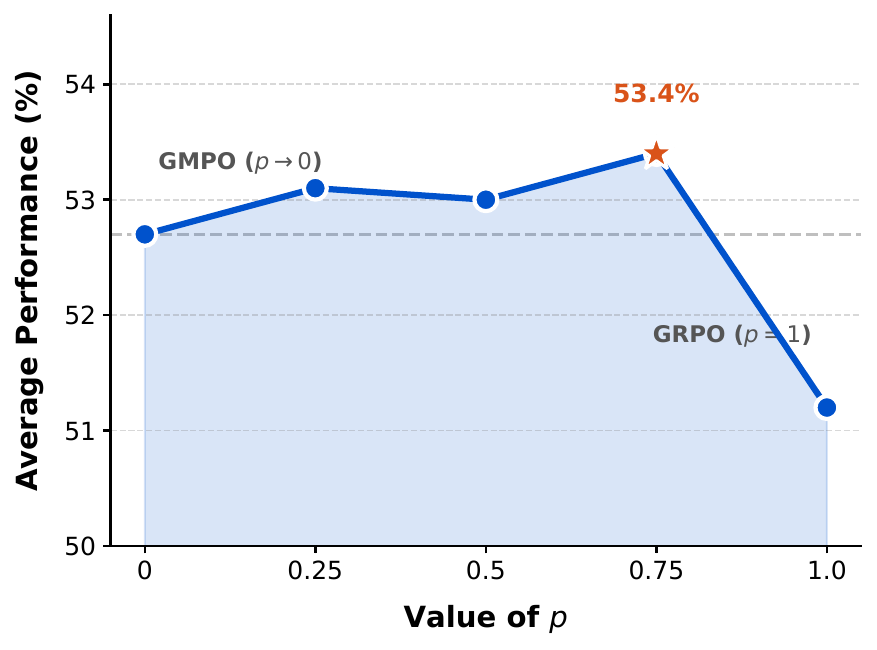}
\centering
\includegraphics[width=0.49\linewidth]{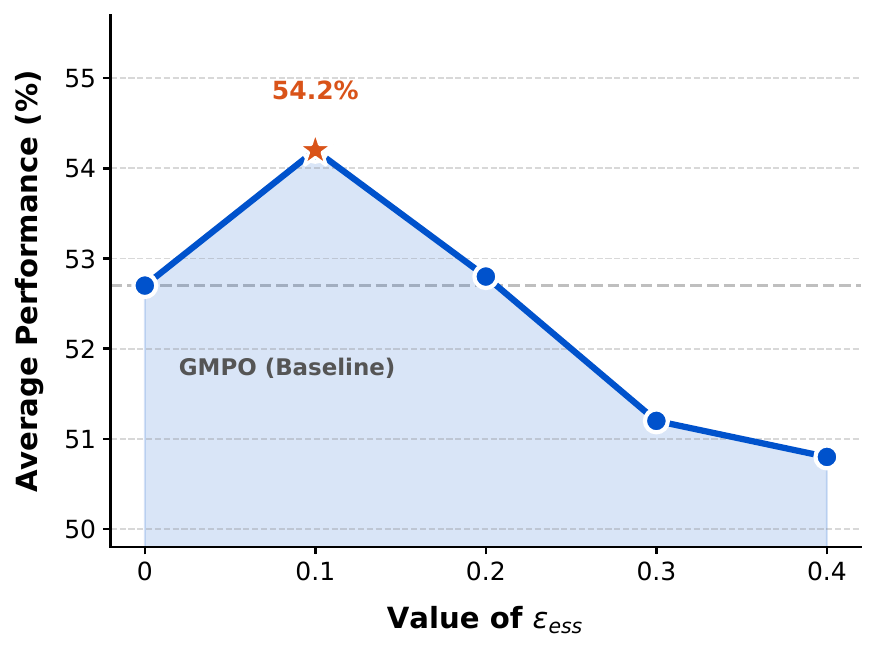}
\vspace{-5mm}
\centering
\caption{Effect of different $p$ and $\epsilon_{ess}$ on average performance on five mathematical reasoning datasets.}
\label{fig:ess}
\vspace{-5mm}
\end{figure}

\begin{figure*}[t]
\centering
\includegraphics[width=0.33\textwidth]{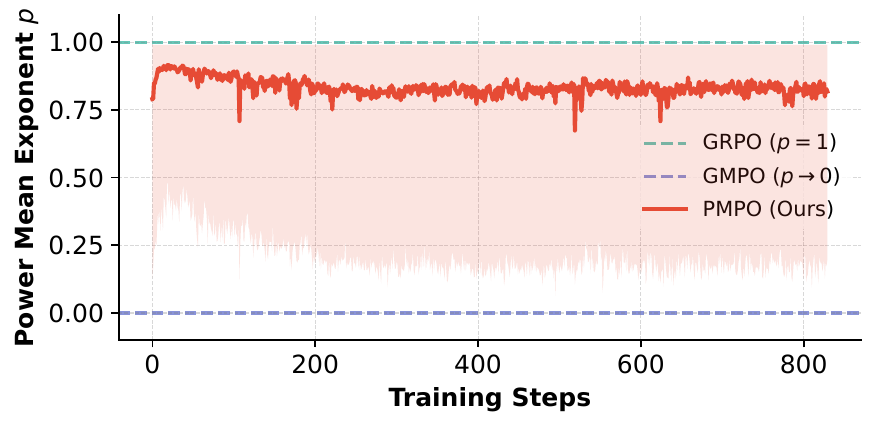}
\centering
\includegraphics[width=0.33\textwidth]{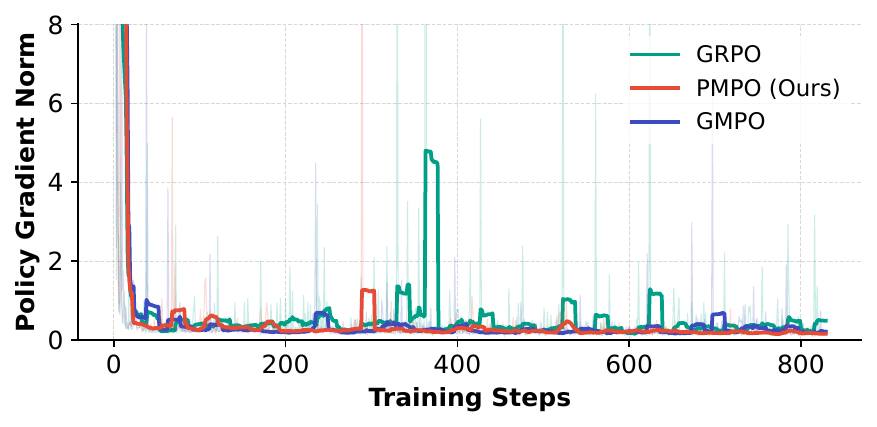}
\vspace{-1mm}
\includegraphics[width=0.33\textwidth]{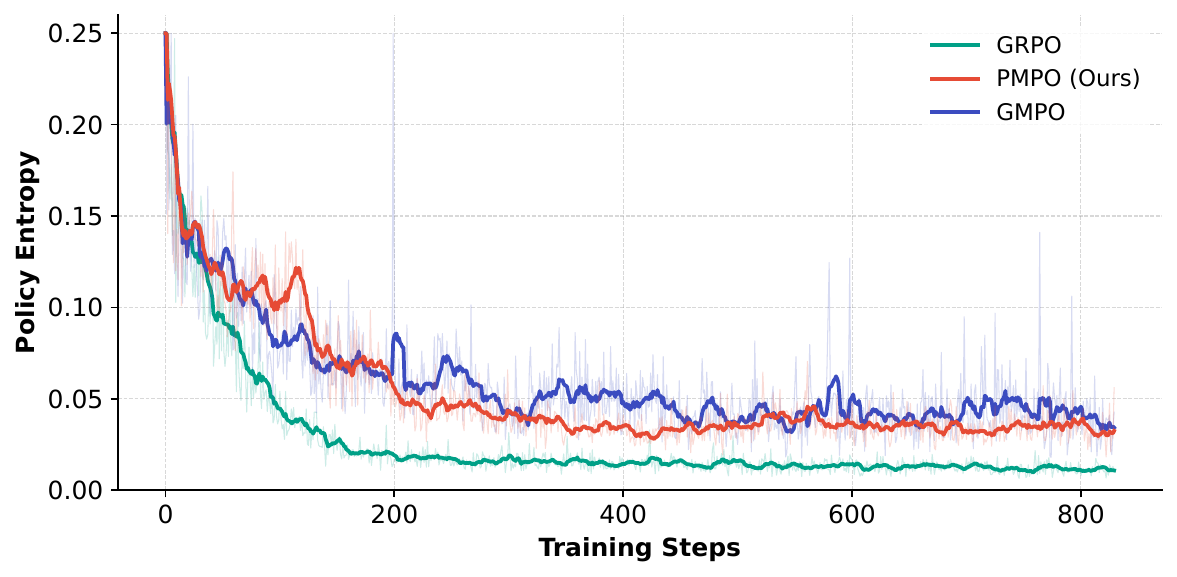}\\
\centering
\includegraphics[width=0.33\textwidth]{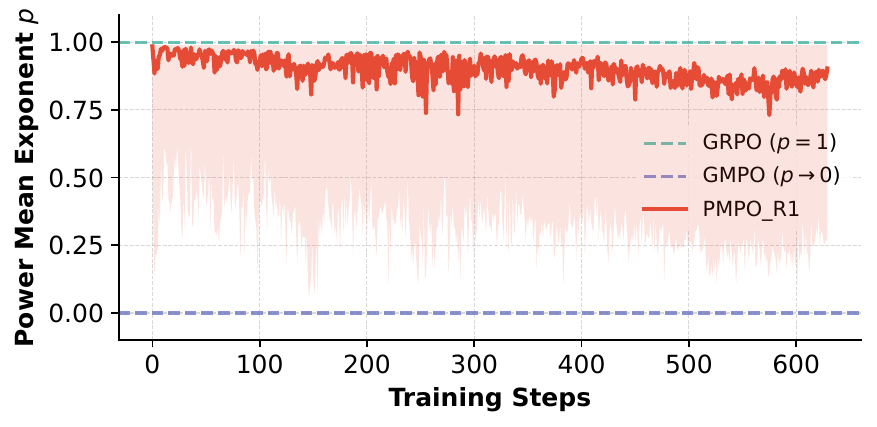}
\centering
\includegraphics[width=0.33\textwidth]{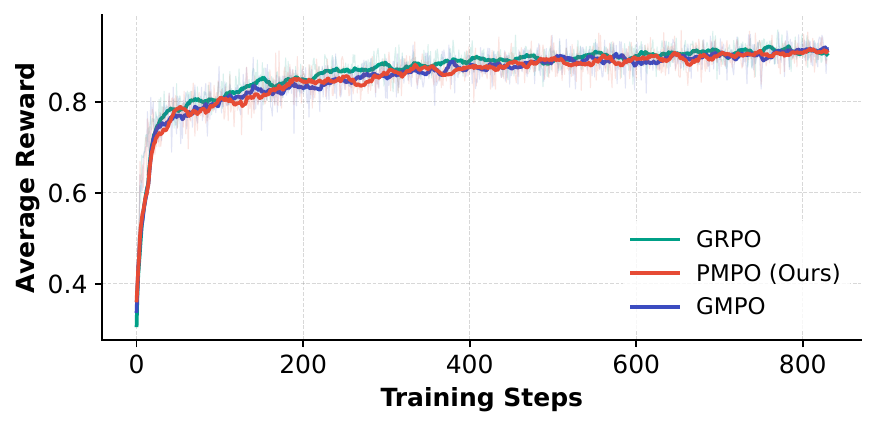}
\centering
\includegraphics[width=0.33\textwidth]{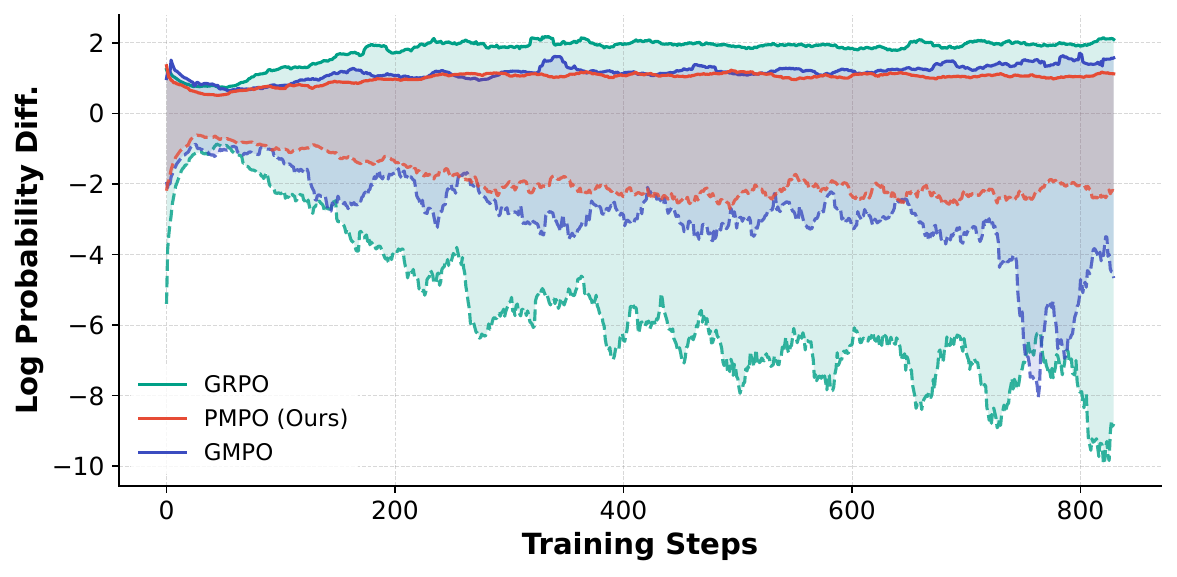}
\vspace{-8mm}
\caption{Training Dynamics. (Top Left and Bottom Left) Comparison of the policy gradient norm evolution across GRPO, GMPO, and PMPO over training steps based on Qwen-Math-7B and DeepSeek-R1-Distill. (Top Middle) The trajectory of the adaptive exponent $p$ (reporting both batch mean and maximum) throughout the training process. (Top Right and Bottom Middle) The curve of (average token-level entropy \& reward) throughout the training. (Bottom Right) The curve of the log-importance ratio $\Delta \ell_t = \log \frac{\pi_\theta(a_t|s_t)}{\pi_{\theta_{\text{old}}}(a_t|s_t)}$. }
\label{fig:p}
\vspace{-4mm}
\end{figure*}
\paragraph{Is Adaptive $p$ Necessary? (Static vs. Dynamic Geometry).}
\label{sec:p}
A natural question is whether the improvement comes from the dynamic mechanism or simply from finding a better fixed $p$ than $0$ or $1$.
We conducted a grid search over fixed $p \in \{0, 0.25, 0.5, 0.75, 1.0\}$.
As illustrated in Fig. \ref{fig:ess} (left), while intermediate values (e.g., $p=0.75$ with 53.4\%) outperform the extreme cases of GMPO ($p \to 0$) and GRPO ($p=1$), they still lag significantly behind our adaptive PMPO (\textbf{54.2\%}).
This gap highlights a crucial insight: the optimal aggregation geometry is not universal but trajectory-dependent.
High-quality trajectories benefit from the sharpness of $p \approx 1$, while noisy trajectories require the smoothing of $p \approx 0$.
A static compromise inevitably under-utilizes safe trajectories while failing to stabilize risky ones, whereas PMPO adaptively navigates this trade-off.
\begin{table}[t]
    \centering
    \caption{We compare PMPO against variants without clipping, with sequence-level clipping, and without length normalization.}
    \label{tab:ablation_clipping}
    \resizebox{\linewidth}{!}{ 
        \begin{tabular}{lcccc}
            \toprule
            & PMPO (Ours) & w/o Clipping & w/ Seq. Clip & w/o Length Norm \\
            \midrule
            Avg. & 54.2 & 50.4 & 54.0 & 53.1 \\
            \bottomrule
        \end{tabular}
    }
    \vspace{-4mm}
\end{table}
\paragraph{Impact of Sequence-level Clipping and Normalization.}
We verify the design choices regarding clipping scope and length normalization in Table~\ref{tab:ablation_clipping}.
Replacing our token-wise clipping with sequence-level clipping (PMPO-seqclip, following GSPO \citep{gspo}) results in a slight performance drop to 54.0\%. This indicates that coarse-grained sequence constraints obscure the fine-grained token instability signals essential for our ESS-matching mechanism.
Similarly, removing length normalization (PMPO-without-norm, following Dr.GRPO \cite{drgrpo}) degrades accuracy to 53.1\%.
While omitting normalization can mitigate length penalty in static arithmetic mean, it introduces scale variance as the geometry shifts. 

\paragraph{Comparison with Heuristic Adaptation Strategies.}
To validate the necessity of our signal-based (clip-aware) adaptation, we explored three intuitive heuristic baselines for determining $p$: Length-based, Entropy-based, and Step-based Annealing (see Appendix \ref{supp:heuristics} for implementation details and Appendix \ref{supp:heuristic_results} for quantitative results).
We observed a divergence in performance: while the Length-based and Schedule-based strategies failed to surpass the static GMPO baseline, the Entropy-based approach slightly outperformed GMPO (53.0\% vs. 52.7\%), validating the premise that aggregation geometry should adapt to signal uncertainty.
However, even the Entropy heuristic still trails our proposed Clip-aware ESS mechanism (54.2\%) by a significant margin.
The limitation lies in the reliance on \emph{indirect proxies}: high entropy or long sequence length does not strictly correlate with optimization instability.
In contrast, our Clip-aware ESS mechanism directly measures \emph{trust-region saturation}, allowing PMPO to maintain aggressive arithmetic updates even for complex or high-entropy trajectories, provided they remain strictly within the trust region.
\paragraph{Evolution of Training Statistics.}
We analyze the training dynamics in Fig. \ref{fig:p}.
The \textbf{entropy curve} shows a desirable pattern: PMPO maintains higher entropy than GMPO in the early stage ($<200$ steps). In RL, a slower entropy decay is generally preferred as it ensures sufficient exploration and prevents premature convergence to local optima. Our adaptive $p$ achieves this by naturally sustaining diversity when the model is unstable.
In the later stage, PMPO's entropy drops lower than GMPO. This is expected behavior: while GMPO's fixed geometric mean forces the policy to remain conservative, PMPO allows the model to assign higher confidence to correct paths once they are reliable.
We also observe that PMPO's importance ratios exhibit significantly smaller fluctuations and a smoother trend compared to baselines (GRPO and GMPO). This enhanced stability stems from our adaptive mechanism, which automatically dampens extreme outliers by lowering $p$ (shifting towards geometric mean), avoiding the variance spikes seen in fixed aggregations.
The reward curve confirms that this lower entropy corresponds to stable convergence rather than overfitting.
Additionally, the adaptive exponent $p$ on DeepSeek-R1-Distill-Qwen-7B mirrors the trajectory on Qwen-Math-7B but stabilizes at slightly higher values. This confirms that our mechanism generalizes: stronger models produce more reliable signals (less clipping), allowing PMPO to adopt more aggressive, arithmetic-like updates.
\vspace{-1mm}
\section{Conclusion}
\vspace{-1mm}
In this work, we bridge the gap between arithmetic and geometric aggregation in group-based RL by introducing \textbf{Power-Mean Policy Optimization (PMPO)}.
We argue that the reliance on a fixed aggregation geometry is a fundamental bottleneck in handling the evolving and heterogeneous nature of mathematical reasoning traces.
To resolve this, we propose the \emph{Clip-aware ESS Matching} mechanism, a principled approach that dynamically adapts the power-mean exponent $p$ based on trust-region saturation.
This mechanism effectively acts as an automatic regulator: it enables aggressive, arithmetic-like updates for reliable signals while seamlessly retreating to conservative, geometric-like stabilization when instability is detected.
Theoretically, we demonstrate that tuning $p$ is equivalent to adjusting the inverse temperature of a softmax attention mechanism over policy gradients.
Extensive experiments across multiple model scales confirm that PMPO consistently outperforms static methods with a new state-of-the-art for online RLHF.

\bibliography{example_paper}

@article{deepseek,
  title={Deepseek-r1: Incentivizing reasoning capability in llms via reinforcement learning},
  author={Guo, Daya and Yang, Dejian and Zhang, Haowei and Song, Junxiao and Zhang, Ruoyu and Xu, Runxin and Zhu, Qihao and Ma, Shirong and Wang, Peiyi and Bi, Xiao and others},
  journal={arXiv preprint arXiv:2501.12948},
  year={2025}
}

@article{grpo,
  title={Deepseekmath: Pushing the limits of mathematical reasoning in open language models},
  author={Shao, Zhihong and Wang, Peiyi and Zhu, Qihao and Xu, Runxin and Song, Junxiao and Bi, Xiao and Zhang, Haowei and Zhang, Mingchuan and Li, YK and Wu, Yang and others},
  journal={arXiv preprint arXiv:2402.03300},
  year={2024}
}

@article{raft,
  title={A minimalist approach to llm reasoning: from rejection sampling to reinforce},
  author={Xiong, Wei and Yao, Jiarui and Xu, Yuhui and Pang, Bo and Wang, Lei and Sahoo, Doyen and Li, Junnan and Jiang, Nan and Zhang, Tong and Xiong, Caiming and others},
  journal={arXiv preprint arXiv:2504.11343},
  year={2025}
}

@article{aapo,
  title={AAPO: Enhance the Reasoning Capabilities of LLMs with Advantage Momentum},
  author={Xiong, Jian and Zhou, Jingbo and Ye, Jingyong and Dou, Dejing},
  journal={arXiv preprint arXiv:2505.14264},
  year={2025}
}

@article{seedgrpo,
  title={Seed-grpo: Semantic entropy enhanced grpo for uncertainty-aware policy optimization},
  author={Chen, Minghan and Chen, Guikun and Wang, Wenguan and Yang, Yi},
  journal={arXiv preprint arXiv:2505.12346},
  year={2025}
}

@article{intuitor,
  title={Learning to reason without external rewards},
  author={Zhao, Xuandong and Kang, Zhewei and Feng, Aosong and Levine, Sergey and Song, Dawn},
  journal={arXiv preprint arXiv:2505.19590},
  year={2025}
}

@article{bnpo,
  title={BNPO: Beta Normalization Policy Optimization},
  author={Xiao, Changyi and Zhang, Mengdi and Cao, Yixin},
  journal={arXiv preprint arXiv:2506.02864},
  year={2025}
}

@article{empo,
  title={Right question is already half the answer: Fully unsupervised llm reasoning incentivization},
  author={Zhang, Qingyang and Wu, Haitao and Zhang, Changqing and Zhao, Peilin and Bian, Yatao},
  journal={arXiv preprint arXiv:2504.05812},
  year={2025}
}

@article{sgrpo,
  title={S-GRPO: Early Exit via Reinforcement Learning in Reasoning Models},
  author={Dai, Muzhi and Yang, Chenxu and Si, Qingyi},
  journal={arXiv preprint arXiv:2505.07686},
  year={2025}
}

@article{cppo,
  title={Cppo: Accelerating the training of group relative policy optimization-based reasoning models},
  author={Lin, Zhihang and Lin, Mingbao and Xie, Yuan and Ji, Rongrong},
  journal={arXiv preprint arXiv:2503.22342},
  year={2025}
}

@article{dler,
  title={Dler: Doing length penalty right-incentivizing more intelligence per token via reinforcement learning},
  author={Liu, Shih-Yang and Dong, Xin and Lu, Ximing and Diao, Shizhe and Liu, Mingjie and Chen, Min-Hung and Yin, Hongxu and Wang, Yu-Chiang Frank and Cheng, Kwang-Ting and Choi, Yejin and others},
  journal={arXiv preprint arXiv:2510.15110},
  year={2025}
}

@misc{qwq,
  title={Qwq-32b: Embracing the power of reinforcement learning},
  author={Team, Qwen},
  year={2025},
  publisher={March}
}

@article{gfpo,
  title={Sample more to think less: Group filtered policy optimization for concise reasoning},
  author={Shrivastava, Vaishnavi and Awadallah, Ahmed and Balachandran, Vidhisha and Garg, Shivam and Behl, Harkirat and Papailiopoulos, Dimitris},
  journal={arXiv preprint arXiv:2508.09726},
  year={2025}
}

@article{grpolambda,
  title={Stable Reinforcement Learning for Efficient Reasoning},
  author={Dai, Muzhi and Liu, Shixuan and Si, Qingyi},
  journal={arXiv preprint arXiv:2505.18086},
  year={2025}
}

@article{grpolead,
  title={Grpo-lead: A difficulty-aware reinforcement learning approach for concise mathematical reasoning in language models},
  author={Zhang, Jixiao and Zuo, Chunsheng},
  journal={arXiv preprint arXiv:2504.09696},
  year={2025}
}

@article{opo,
  title={On-Policy RL with Optimal Reward Baseline},
  author={Hao, Yaru and Dong, Li and Wu, Xun and Huang, Shaohan and Chi, Zewen and Wei, Furu},
  journal={arXiv preprint arXiv:2505.23585},
  year={2025}
}

@article{repo,
  title={RePO: Replay-Enhanced Policy Optimization},
  author={Li, Siheng and Zhou, Zhanhui and Lam, Wai and Yang, Chao and Lu, Chaochao},
  journal={arXiv preprint arXiv:2506.09340},
  year={2025}
}

@article{srpo,
  title={Srpo: A cross-domain implementation of large-scale reinforcement learning on llm},
  author={Zhang, Xiaojiang and Wang, Jinghui and Cheng, Zifei and Zhuang, Wenhao and Lin, Zheng and Zhang, Minglei and Wang, Shaojie and Cui, Yinghan and Wang, Chao and Peng, Junyi and others},
  journal={arXiv preprint arXiv:2504.14286},
  year={2025}
}

@article{ess,
  title   = {Sequential Imputations and {B}ayesian Missing Data Problems},
  author  = {Kong, Augustine and Liu, Jun S. and Wong, Wing Hung},
  journal = {Journal of the American Statistical Association},
  volume  = {89},
  number  = {425},
  pages   = {278--288},
  year    = {1994}
}

@article{qwen3,
  title={Qwen3 technical report},
  author={Yang, An and Li, Anfeng and Yang, Baosong and Zhang, Beichen and Hui, Binyuan and Zheng, Bo and Yu, Bowen and Gao, Chang and Huang, Chengen and Lv, Chenxu and others},
  journal={arXiv preprint arXiv:2505.09388},
  year={2025}
}

@article{aime24,
  title={Numinamath: The largest public dataset in ai4maths with 860k pairs of competition math problems and solutions},
  author={Li, Jia and Beeching, Edward and Tunstall, Lewis and Lipkin, Ben and Soletskyi, Roman and Huang, Shengyi and Rasul, Kashif and Yu, Longhui and Jiang, Albert Q and Shen, Ziju and others},
  journal={Hugging Face repository},
  volume={13},
  number={9},
  pages={9},
  year={2024}
}

@article{qwen2.5,
  author       = {An Yang and
                  Baosong Yang and
                  Beichen Zhang and
                  Binyuan Hui and
                  Bo Zheng and
                  Bowen Yu and
                  Chengyuan Li and
                  Dayiheng Liu and
                  Fei Huang and
                  Haoran Wei and
                  Huan Lin and
                  Jian Yang and
                  Jianhong Tu and
                  Jianwei Zhang and
                  Jianxin Yang and
                  Jiaxi Yang and
                  Jingren Zhou and
                  Junyang Lin and
                  Kai Dang and
                  Keming Lu and
                  Keqin Bao and
                  Kexin Yang and
                  Le Yu and
                  Mei Li and
                  Mingfeng Xue and
                  Pei Zhang and
                  Qin Zhu and
                  Rui Men and
                  Runji Lin and
                  Tianhao Li and
                  Tingyu Xia and
                  Xingzhang Ren and
                  Xuancheng Ren and
                  Yang Fan and
                  Yang Su and
                  Yichang Zhang and
                  Yu Wan and
                  Yuqiong Liu and
                  Zeyu Cui and
                  Zhenru Zhang and
                  Zihan Qiu},
  title        = {Qwen2.5 Technical Report},
  journal      = {CoRR},
  volume       = {abs/2412.15115},
  year         = {2024},
  url          = {https://doi.org/10.48550/arXiv.2412.15115},
  doi          = {10.48550/ARXIV.2412.15115},
  eprinttype    = {arXiv},
  eprint       = {2412.15115},
  bibsource    = {dblp computer science bibliography, https://dblp.org}
}

@inproceedings{oly,
  title={Olympiadbench: A challenging benchmark for promoting agi with olympiad-level bilingual multimodal scientific problems},
  author={He, Chaoqun and Luo, Renjie and Bai, Yuzhuo and Hu, Shengding and Thai, Zhen and Shen, Junhao and Hu, Jinyi and Han, Xu and Huang, Yujie and Zhang, Yuxiang and others},
  booktitle={Proceedings of the 62nd Annual Meeting of the Association for Computational Linguistics (Volume 1: Long Papers)},
  pages={3828--3850},
  year={2024}
}

@article{math,
  title={Measuring mathematical problem solving with the math dataset},
  author={Hendrycks, Dan and Burns, Collin and Kadavath, Saurav and Arora, Akul and Basart, Steven and Tang, Eric and Song, Dawn and Steinhardt, Jacob},
  journal={arXiv preprint arXiv:2103.03874},
  year={2021}
}

@article{minerva,
  title={Solving quantitative reasoning problems with language models},
  author={Lewkowycz, Aitor and Andreassen, Anders and Dohan, David and Dyer, Ethan and Michalewski, Henryk and Ramasesh, Vinay and Slone, Ambrose and Anil, Cem and Schlag, Imanol and Gutman-Solo, Theo and others},
  journal={Advances in neural information processing systems},
  volume={35},
  pages={3843--3857},
  year={2022}
}

@article{gmpo,
  title={Geometric-mean policy optimization},
  author={Zhao, Yuzhong and Liu, Yue and Liu, Junpeng and Chen, Jingye and Wu, Xun and Hao, Yaru and Lv, Tengchao and Huang, Shaohan and Cui, Lei and Ye, Qixiang and others},
  journal={arXiv preprint arXiv:2507.20673},
  year={2025}
}

@article{gspo,
  title={Group sequence policy optimization},
  author={Zheng, Chujie and Liu, Shixuan and Li, Mingze and Chen, Xiong-Hui and Yu, Bowen and Gao, Chang and Dang, Kai and Liu, Yuqiong and Men, Rui and Yang, An and others},
  journal={arXiv preprint arXiv:2507.18071},
  year={2025}
}

@article{prime,
  title={Process reinforcement through implicit rewards},
  author={Cui, Ganqu and Yuan, Lifan and Wang, Zefan and Wang, Hanbin and Zhang, Yuchen and Chen, Jiacheng and Li, Wendi and He, Bingxiang and Fan, Yuchen and Yu, Tianyu and others},
  journal={arXiv preprint arXiv:2502.01456},
  year={2025}
}

@article{gpg,
  title={Gpg: A simple and strong reinforcement learning baseline for model reasoning},
  author={Chu, Xiangxiang and Huang, Hailang and Zhang, Xiao and Wei, Fei and Wang, Yong},
  journal={arXiv preprint arXiv:2504.02546},
  year={2025}
}

@article{eurus,
  title={Advancing llm reasoning generalists with preference trees},
  author={Yuan, Lifan and Cui, Ganqu and Wang, Hanbin and Ding, Ning and Wang, Xingyao and Deng, Jia and Shan, Boji and Chen, Huimin and Xie, Ruobing and Lin, Yankai and others},
  journal={arXiv preprint arXiv:2404.02078},
  year={2024}
}

@article{open,
  title={Open-reasoner-zero: An open source approach to scaling up reinforcement learning on the base model},
  author={Hu, Jingcheng and Zhang, Yinmin and Han, Qi and Jiang, Daxin and Zhang, Xiangyu and Shum, Heung-Yeung},
  journal={arXiv preprint arXiv:2503.24290},
  year={2025}
}

@article{simplerl,
  title={Simplerlzoo: Investigating and taming zero reinforcement learning for open base models in the wild. CoRR, abs/2503.18892, 2025. doi: 10.48550},
  author={Zeng, Weihao and Huang, Yuzhen and Liu, Qian and Liu, Wei and He, Keqing and Ma, Zejun and He, Junxian},
  journal={arXiv preprint ARXIV.2503.18892}
}

@article{dapo,
  title={Dapo: An open-source llm reinforcement learning system at scale},
  author={Yu, Qiying and Zhang, Zheng and Zhu, Ruofei and Yuan, Yufeng and Zuo, Xiaochen and Yue, Yu and Dai, Weinan and Fan, Tiantian and Liu, Gaohong and Liu, Lingjun and others},
  journal={arXiv preprint arXiv:2503.14476},
  year={2025}
}

@article{drgrpo,
  title={Understanding r1-zero-like training: A critical perspective},
  author={Liu, Zichen and Chen, Changyu and Li, Wenjun and Qi, Penghui and Pang, Tianyu and Du, Chao and Lee, Wee Sun and Lin, Min},
  journal={arXiv preprint arXiv:2503.20783},
  year={2025}
}

@article{ppo,
  title={Proximal policy optimization algorithms},
  author={Schulman, John and Wolski, Filip and Dhariwal, Prafulla and Radford, Alec and Klimov, Oleg},
  journal={arXiv preprint arXiv:1707.06347},
  year={2017}
}
\bibliographystyle{icml2025}

\newpage
\appendix
\onecolumn
\section{Proof of Power Mean Monotonicity}
\label{app:proof_monotonicity}

In this section, we provide the proof for the monotonicity of the power mean with respect to its exponent $p$.

\begin{theorem}
Let $\{x_t\}_{t=1}^n$ be a set of positive real numbers ($x_t > 0$). For any two non-zero real numbers $p_1$ and $p_2$ such that $p_1 < p_2$, the following inequality holds:
\begin{equation}
    M_{p_1}(\{x_t\}) \le M_{p_2}(\{x_t\}),
\end{equation}
where $M_p(\{x_t\}) = \left( \frac{1}{n} \sum_{t=1}^n x_t^p \right)^{\frac{1}{p}}$. Equality holds if and only if all $x_t$ are equal.
\end{theorem}

\begin{proof}
The proof relies on \textbf{Jensen's Inequality}. Recall that for a convex function $\phi$ and a random variable $Y$, $\phi(\mathbb{E}[Y]) \le \mathbb{E}[\phi(Y)]$.

\textbf{Case 1: $0 < p_1 < p_2$.} \\
Let $Y$ be a discrete random variable taking values $x_t^{p_1}$ with uniform probability $\frac{1}{n}$. Let $\alpha = \frac{p_2}{p_1}$. Since $0 < p_1 < p_2$, we have $\alpha > 1$.
Consider the function $\phi(u) = u^\alpha$ for $u > 0$. The second derivative is $\phi''(u) = \alpha(\alpha-1)u^{\alpha-2}$. Since $\alpha > 1$ and $u > 0$, $\phi''(u) > 0$, implying that $\phi(u)$ is strictly convex.

Applying Jensen's Inequality:
\begin{equation}
    \phi\left( \frac{1}{n} \sum_{t=1}^n x_t^{p_1} \right) \le \frac{1}{n} \sum_{t=1}^n \phi\left( x_t^{p_1} \right).
\end{equation}
Substituting $\phi(u) = u^{\frac{p_2}{p_1}}$:
\begin{equation}
    \left( \frac{1}{n} \sum_{t=1}^n x_t^{p_1} \right)^{\frac{p_2}{p_1}} \le \frac{1}{n} \sum_{t=1}^n \left( x_t^{p_1} \right)^{\frac{p_2}{p_1}} = \frac{1}{n} \sum_{t=1}^n x_t^{p_2}.
\end{equation}
Since $p_2 > 0$, the function $z \mapsto z^{\frac{1}{p_2}}$ is increasing. Raising both sides to the power of $\frac{1}{p_2}$ preserves the inequality:
\begin{equation}
    \left[ \left( \frac{1}{n} \sum_{t=1}^n x_t^{p_1} \right)^{\frac{p_2}{p_1}} \right]^{\frac{1}{p_2}} \le \left( \frac{1}{n} \sum_{t=1}^n x_t^{p_2} \right)^{\frac{1}{p_2}}.
\end{equation}
Simplifying the left-hand side:
\begin{equation}
    \left( \frac{1}{n} \sum_{t=1}^n x_t^{p_1} \right)^{\frac{1}{p_1}} \le \left( \frac{1}{n} \sum_{t=1}^n x_t^{p_2} \right)^{\frac{1}{p_2}}.
\end{equation}
Thus, $M_{p_1}(\{x_t\}) \le M_{p_2}(\{x_t\})$.

\textbf{Case 2: $p_1 < 0 < p_2$.} \\
From Case 1, we know that $M_{p}$ is increasing for $p>0$. Similarly, one can prove it is increasing for $p<0$. The discontinuity at $p=0$ is handled by the limit definition of the Geometric Mean. It is a known result that:
\begin{equation}
    \lim_{p \to 0} M_p(\{x_t\}) = \left( \prod_{t=1}^n x_t \right)^{\frac{1}{n}} = G(\{x_t\}).
\end{equation}
Using the weighted AM-GM inequality (or simply GM $\le$ AM), we know $M_0 \le M_{p_2}$ for $p_2 > 0$. Similarly, HM $\le$ GM (Harmonic Mean $\le$ Geometric Mean) implies $M_{p_1} \le M_0$ for $p_1 < 0$. By transitivity, the monotonicity holds across zero.

For the purpose of PMPO, where $p$ is typically adapted within $(0, 1]$, Case 1 is the primary justification.
\end{proof}

\section{Theoretical Analysis of ESS-Matching}
\label{app:theoretical_analysis}

In this section, we provide a comprehensive analysis of the proposed Clip-aware ESS Matching mechanism. We strictly motivate the design choice via scale invariance, prove the mathematical validity of the solver, and visually illustrate the mechanism.

\subsection{Motivation: The Necessity of Scale Invariance}
\label{app:scale_invariance}

A natural question arises: \textit{Why not map the clipping fraction directly to the exponent $p$ (e.g., via a linear decay $p = 1 - f_{\mathrm{clip}}$)?} We argue that such a direct mapping is suboptimal because $p$ lacks \textbf{scale invariance}.
Recall from Eq.~(\ref{gradient}) in the main text that the gradient of the power mean $\hat{r}_p$ behaves as a softmax function over the log-probability changes $\Delta \ell_t$, with an inverse temperature governed by $p$:
\begin{equation}
    \frac{\partial \hat{r}_p}{\partial \Delta \ell_j} \propto \mathrm{softmax}\left( \frac{\Delta \ell}{1/p} \right)_j = \frac{\exp(p \Delta \ell_j)}{\sum_k \exp(p \Delta \ell_k)}.
\end{equation}
The "sharpness" or concentration of this gradient distribution depends on the product $p \cdot \Delta \ell_t$. This implies that for a fixed $p$, the effective behavior of the aggregation shifts dramatically depending on the magnitude of the log-probability differences: 1) If $\Delta \ell_t$ values are large (e.g., early in training or for high-confidence models), a moderate $p$ (e.g., 0.5) acts aggressively, creating a highly peaked distribution. 2)If $\Delta \ell_t$ values are small (e.g., near convergence), the same $p=0.5$ acts conservatively, approaching a uniform distribution.
Therefore, a static mapping from clipping to $p$ would lead to inconsistent update behaviors across different training stages and prompts.
By contrast, the \textbf{Effective Sample Size (ESS)} is inherently scale-invariant. It measures the sparsity of the normalized weight distribution $w$, independent of the absolute magnitude of the input logits. By defining our target in terms of ESS, we ensure that the \emph{concentration of the policy update} is consistently regulated by the signal reliability (clipping), regardless of the evolving scale of the latent log-probability differences. The solver then finds the specific $p$ required to achieve this desired sparsity given the current batch's statistics.

To empirically validate this theoretical analysis, we implemented a baseline variant named \textbf{PMPO-Direct}, which uses a naive linear mapping rule: $p_k = 1 - f_{\mathrm{clip}}^{(k)}$. As shown in Table~\ref{tab:app_direct_comparison}, while PMPO-Direct (52.8\%) marginally outperforms the static GMPO baseline, it significantly underperforms our ESS-based PMPO (54.2\%). This performance gap confirms that decoupling the geometry adaptation from the signal scale via ESS is critical for stability and efficiency.

\begin{table}[htbp]
    \centering
    \caption{\textbf{Impact of Scale Invariance.} Comparison between naive Direct Mapping and our proposed ESS Matching on Qwen2.5-Math-7B. The ESS mechanism significantly outperforms the linear heuristic.}
    \label{tab:app_direct_comparison}
    \setlength{\tabcolsep}{8pt}
    \begin{tabular}{lcccc}
        \toprule
        \textbf{Method} & \textbf{Mechanism} & \textbf{Scale Invariant?} & \textbf{Avg. Accuracy} & \textbf{$\Delta$ vs. GMPO} \\
        \midrule
        GMPO & Fixed ($p \to 0$) & N/A & 52.7 & - \\
        PMPO-Direct & Linear Mapping ($p \propto 1 - f_{\mathrm{clip}}$) & \xmark & 53.2 & +0.5 \\
        \textbf{PMPO (Ours)} & \textbf{Clip-aware ESS Matching} & \cmark & \textbf{54.2} & \textbf{+1.5} \\
        \bottomrule
    \end{tabular}
\end{table}

\subsection{Formal Proof of ESS Monotonicity}
\label{app:proof}
The validity of using the bisection method to solve for $p$ relies on the strict monotonicity of the ESS function.

\begin{proposition}
Let $z_t = \tilde{\Delta}\ell_t$ be the clipped log-probability differences for $t=1,\dots,n$. The normalized Effective Sample Size function $\mathrm{ESS}_{\mathrm{norm}}(p)$ defined by the induced weights $w_t(p) = \mathrm{softmax}(p z_t)$ is monotonically decreasing for $p \ge 0$. Specifically, $\frac{\partial \mathrm{ESS}_{\mathrm{norm}}(p)}{\partial p} \le 0$, with strict inequality unless all $z_t$ are equal.
\end{proposition}

\begin{proof}
Recall the definition of the unnormalized ESS:
\begin{equation}
    E(p) = \frac{1}{\sum_{t=1}^n w_t(p)^2},   \text{where } w_t(p) = \frac{e^{p z_t}}{\sum_{j=1}^n e^{p z_j}}.
\end{equation}
To analyze the monotonicity, it suffices to analyze the denominator $S(p) = \sum_{t=1}^n w_t(p)^2$. If $S(p)$ is monotonically increasing, then $E(p)$ is monotonically decreasing.

Let us compute the derivative of $w_t(p)$ with respect to $p$:
\begin{equation}
    \frac{\partial w_t}{\partial p} = w_t (z_t - \bar{z}_p),   \text{where } \bar{z}_p = \sum_{k=1}^n w_k z_k = \mathbb{E}_w[z].
\end{equation}
Now, differentiate $S(p)$:
\begin{equation}
\begin{aligned}
    \frac{\partial S(p)}{\partial p} &= \sum_{t=1}^n 2 w_t \frac{\partial w_t}{\partial p} \\
    &= 2 \sum_{t=1}^n w_t^2 (z_t - \bar{z}_p) \\
    &= 2 \left( \sum_{t=1}^n w_t^2 z_t - \left(\sum_{t=1}^n w_t^2\right) \bar{z}_p \right) \\
    &= 2 S(p) \left( \frac{\sum_{t=1}^n w_t^2 z_t}{\sum_{t=1}^n w_t^2} - \sum_{k=1}^n w_k z_k \right).
\end{aligned}
\end{equation}
Let us define a new probability distribution (the ``squared'' distribution) $q$, where $q_t = \frac{w_t^2}{S(p)}$. Note that since large $w_t$ become relatively larger when squared, $q$ places more mass on indices with larger $z_t$ compared to $w$.
The term in the parenthesis becomes:
\begin{equation}
    \Delta_z = \mathbb{E}_q[z] - \mathbb{E}_w[z].
\end{equation}
Since $w_t \propto e^{p z_t}$, $w_t$ is a strictly increasing function of $z_t$ (for $p>0$). The distribution $q_t \propto w_t^2 \propto e^{2p z_t}$ can be seen as the distribution resulting from doubling the temperature parameter (from $p$ to $2p$).
It is a known property of the exponential family that the derivative of the expected sufficient statistic with respect to the natural parameter is the variance of the sufficient statistic. Specifically, $\frac{\partial}{\partial \theta} \mathbb{E}_{p_\theta}[z] = \operatorname{Var}_{p_\theta}[z] \ge 0$. Since we are comparing the expectation under parameter $2p$ (distribution $q$) versus $p$ (distribution $w$) and $2p > p$ (for $p>0$), the strict monotonicity implies:
\begin{equation}
\mathbb{E}_q[z] \ge \mathbb{E}_w[z].
\end{equation} Thus, for a ``sharper'' distribution $q$ (which corresponds to parameter $2p$) versus $w$ (parameter $p$), the expectation satisfies:
\begin{equation}
    \mathbb{E}_q[z] \ge \mathbb{E}_w[z].
\end{equation}
Equality holds if and only if all $z_t$ are identical (in which case the distribution is uniform regardless of $p$).

Therefore, $\frac{\partial S(p)}{\partial p} \ge 0$, which implies $S(p)$ is increasing. Consequently, $\mathrm{ESS}(p) = 1/S(p)$ is monotonically decreasing. This guarantees that for any target $\eta^\star \in (\frac{1}{n}, 1]$, there exists a unique $p$, validating the use of the bisection method.
\end{proof}

\subsection{Visual Illustration} 
\label{app:visual_illustration}
\begin{figure}[htbp]
    \centering
    \includegraphics[width=0.6\linewidth]{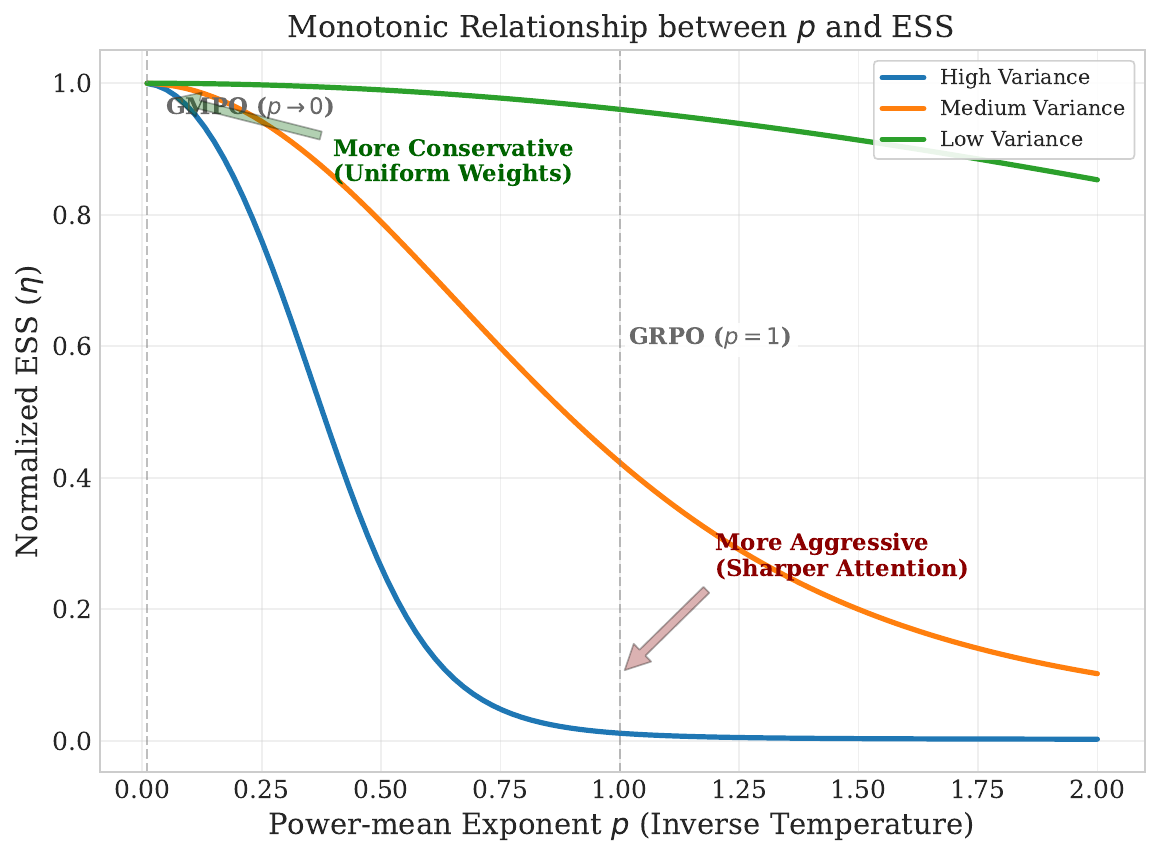}
    \caption{\textbf{Visual Illustration of the Monotonic Relationship between $p$ and ESS.} We simulate token-level log-probability differences $\tilde{\Delta}\ell$ with different variances. The curves demonstrate that increasing $p$ consistently reduces the Effective Sample Size. The varying slopes highlight that the sensitivity of ESS to $p$ depends on the signal variance, necessitating the proposed ESS-matching strategy.}
    \label{fig:ess_monotonicity}
\end{figure}
To complement the formal proof, we provide a numerical demonstration to visually illustrate the relationship between $p$ and ESS under different signal conditions. As analyzed in the proof, $p$ acts as the inverse temperature, governing the sharpness of the weight distribution.
Fig. \ref{fig:ess_monotonicity} plots the ESS curves using simulated log-probability distributions with varying variances (representing ``decisive'', ``typical'', and ``noisy'' reasoning steps). The simulation confirms that the strictly monotonic relationship holds regardless of the data distribution: \textbf{1) Monotonicity:} In all scenarios, increasing $p$ (more aggressive) strictly decreases the Effective Sample Size (more concentrated weights). \textbf{2) Scale Adaptation:} High-variance trajectories (blue line) exhibit a sharp drop in ESS even at moderate $p$, whereas low-variance trajectories (green line) require a much larger $p$ to reduce ESS. This visually justifies the necessity of our adaptive mechanism: a fixed $p$ (e.g., $p=1$) would yield inconsistent aggregation behaviors across different

\section{Analysis on Relaxing the Upper Bound ($p > 1$)}
\label{app:unbounded_p}

In our main experiments, we restricted $p < 1$. A natural question is whether allowing larger upper bound threshold provides any benefit.
Mathematically, while $p=1$ (Arithmetic Mean) considers all tokens linearly, using $p > 1$ disproportionately assigns higher weights to tokens with larger importance ratios. As $p$ increases, the aggregation logic shifts towards a \emph{maximum} operation, where the gradient is dominated by the single token with the largest probability shift, effectively ignoring the rest of the sequence.
To test this, we relaxed the upper bound $p_{\max}$ to $2.0$ and $4.0$ on Qwen2.5-Math-7B. We observed that the adaptive solver rarely reached these high values. Even with the bound set to 4.0, the mean value of $p$ stabilized around $\mathbf{1.2}$. This indicates that even a small increase in $p$ drastically reduces the Effective Sample Size (ESS), quickly meeting the sparsity target without needing larger exponents.
Despite the small shift in the average $p$ (from 0.8 to $\approx 1.2$), Tab. \ref{tab:p_unbounded} shows a consistent drop in performance. 1) The degradation confirms that $p=1$ is a critical boundary. Any $p > 1$ amplifies the influence of outlier tokens (those with sudden high probability ratios). In mathematical reasoning, a correct answer depends on the coherence of the entire chain, not just a few isolated tokens with high scores. 2) By over-weighting a few tokens, $p > 1$ reduces the effective number of tokens contributing to the gradient. This increases the variance of the update, making the training less stable compared to the unbiased arithmetic mean.
Therefore, we conclude that $p=1$ is the optimal upper limit for this task.
\begin{table}[htbp]
    \centering
    \caption{\textbf{Effect of $p > 1$.} Performance degrades when allowing $p$ to exceed 1.0, even though the average $p$ remains low ($\approx 1.2$).}
    \label{tab:p_unbounded}
    \setlength{\tabcolsep}{5pt}
    \begin{tabular}{lccccccc}
        \toprule
        \textbf{Config} & \textbf{Range} & \textbf{AIME24} & \textbf{AMC} & \textbf{MATH500} & \textbf{Minerva} & \textbf{Oly.} & \textbf{Avg.} \\
        \midrule
        \textbf{PMPO (Default)} & $p \in [0.01, 1.0]$ & \textbf{36.7} & \textbf{68.7} & \textbf{83.8} & \textbf{34.9} & \textbf{46.7} & \textbf{54.2} \\
        PMPO (Relaxed) & $p \in [0.01, 2.0]$ & 33.3 & 66.3 & 82.8 & 33.5 & 45.1 & 52.2 \\
        PMPO (Extreme) & $p \in [0.01, 4.0]$ & 33.3 & 65.1 & 82.6 & 33.2 & 44.8 & 51.8 \\
        \bottomrule
    \end{tabular}
\end{table}
\section{Alternative Heuristics for $p$-Adaptation}
\label{supp:heuristics}

In the early stages of our research, we investigated several heuristic strategies for adapting the power-mean exponent $p$ based on trajectory statistics or training progress. Although our final \emph{Clip-aware ESS} method proved superior, we document these heuristics here for completeness. The implementation details for these variants are as follows:

\subsection{Sequence Length Adaptation.}
Motivated by the observation that longer reasoning chains often accumulate more variance and are prone to "outlier" tokens, we hypothesized that longer sequences require more conservative aggregation (lower $p$, closer to geometric mean).
We maintain an exponential moving average (EMA) of the sequence lengths $\mu_L$ and their variance $\sigma_L^2$. For a given trajectorywith length $L_k$, we compute its Z-score $z_k = (L_k - \mu_L)/\sigma_L$. The exponent $p$ is determined via a sigmoid mapping:
\begin{equation}
    p_k = p_{\min} + (p_{\max} - p_{\min}) \cdot \sigma(-\alpha \cdot z_k),
\end{equation}
where $\alpha$ is a sensitivity hyperparameter. Under this formulation, trajectories significantly longer than the running average ($z_k > 0$) result in a smaller $p$, enforcing stricter geometric-like regularization.

\subsection{Entropy-based Adaptation.}
We explored using the policy's Shannon entropy as a proxy for uncertainty. We hypothesize that high-entropy trajectories (indicating model uncertainty) require conservative updates, while low-entropy trajectories (high confidence) support aggressive arithmetic updates.
Similar to the length-based approach, we compute the mean per-token entropy $H_k$ for the trajectoryand normalize it against the historical EMA statistics to obtain a Z-score $z_{H}$. The exponent is adapted as:
\begin{equation}
    p_k = p_{\min} + (p_{\max} - p_{\min}) \cdot \sigma(-\alpha \cdot z_{H}).
\end{equation}
High uncertainty ($z_H > 0$) drives $p \to p_{\min}$ (Geometric regime), while high confidence ($z_H < 0$) drives $p \to p_{\max}$ (Arithmetic regime).

\subsection{Time-Dependent Annealing (Step Schedule).}
We also evaluated a deterministic annealing schedule that ignores trajectory-specific features and depends solely on the training step $t$. We utilized a cosine decay schedule to transition from an aggressive arithmetic geometry to a conservative geometric geometry over the course of training:
\begin{equation}
    p_t = p_{\min} + (p_{\max} - p_{\min}) \cdot \frac{1}{2}\left(1 + \cos\left(\frac{t}{T_{\text{total}}}\pi\right)\right).
\end{equation}
This schedule initializes training with $p \approx 1$ (GRPO-like) to encourage rapid initial learning and anneals towards $p \approx 0$ (GMPO-like) to stabilize convergence in later stages.

\section{Extended Experimental Results on Heuristic Strategies}
\label{supp:heuristic_results}

In this section, we present a detailed performance comparison between our proposed Clip-aware ESS mechanism and the alternative heuristic strategies introduced in Appendix~\ref{supp:heuristics}.
Table~\ref{tab:heuristic_results} summarizes the results on Qwen2.5-Math-7B.
We observe that heuristics based on surface-level statistics yield mixed results, highlighting the difficulty of manually designing the adaptation rule.

\paragraph{Length-based Adaptation (PMPO-Length):} Conversely, penalizing longer sequences leads to a performance degradation (51.9\%) compared to GMPO. We attribute this to the fact that mathematical reasoning often requires long Chain-of-Thought (CoT) derivations to reach correct solutions. Blindly forcing a conservative geometry ($p \to 0$) for long trajectories incorrectly dampens the learning signal from valid, complex reasoning paths, effectively penalizing the model for "thinking hard."

\paragraph{Entropy-based Adaptation (PMPO-Entropy):} Among the heuristics, the entropy-based approach performs best, achieving an average accuracy of \textbf{53.0\%}, which slightly outperforms the static GMPO baseline (52.7\%). This suggests that token-level entropy is a reasonable proxy for model uncertainty: high-entropy trajectories likely contain "risky" steps that benefit from the conservative geometric aggregation ($p \to 0$), while low-entropy (confident) trajectories can safely utilize the more aggressive arithmetic mean.

\paragraph{Step Schedule (PMPO-Schedule):} The deterministic annealing schedule (52.1\%) also fails to surpass the GMPO baseline. This result indicates that the optimal aggregation geometry is highly \emph{data-dependent} rather than training-stage dependent. A rigid schedule forces the model to be aggressive or conservative regardless of the actual quality of the trajectories, leading to suboptimal updates.

\paragraph{Superiority of Clip-aware ESS.}
In contrast, our proposed \textbf{PMPO (Adaptive)} achieves \textbf{54.2\%}, significantly outperforming all heuristic variants.
Unlike length or entropy, which are indirect proxies, the clipping fraction is a \emph{direct measurement} of trust-region saturation.
It allows PMPO to distinguish between "long but reliable" trajectories (keeping $p$ high) and "short but unstable" ones (lowering $p$), thereby maximizing sample efficiency without sacrificing stability.

\begin{table*}[!t]
    \centering
    \caption{\textbf{Comparison with Heuristic Adaptation Strategies.} Average accuracy (Pass@1) on mathematical reasoning benchmarks (Qwen2.5-Math-7B).}
    \label{tab:heuristic_results}
    \setlength{\tabcolsep}{6pt}
    \begin{tabular}{lccccccc}
        \toprule
        \textbf{Method} & \textbf{Mechanism} & \textbf{AIME24} & \textbf{AMC} & \textbf{MATH500} & \textbf{Minerva} & \textbf{Oly.} & \textbf{Avg.} \\
        \midrule
        \rowcolor{gray!10} \textbf{Baselines} & & & & & & & \\
        GMPO ($p \to 0$) & Fixed & 43.3 & 61.4 & 82.0 & 33.5 & 43.6 & 52.7 \\
        GRPO ($p=1$)     & Fixed & 40.0 & 59.0 & 83.4 & 32.4 & 41.3 & 51.2 \\
        \midrule
        \rowcolor{gray!10} \textbf{Heuristics} & & & & & & & \\
        PMPO-Length   & $L_k \uparrow \Rightarrow p \downarrow$ & 33.3 & 65.1 & 82.6 & 32.0 & 46.4 & 51.9 \\
        PMPO-Entropy  & $H_k \uparrow \Rightarrow p \downarrow$ & 33.3 & 67.4 & 83.0 & 35.8 & 45.5 & 53.0 \\
        PMPO-Schedule & $t \uparrow \Rightarrow p \downarrow$   & 40.0 & 60.2 & 82.6 & 34.6 & 43.5 & 52.1 \\
        \midrule
        \rowcolor{gray!10} \textbf{Ours} & & & & & & & \\
        \textbf{PMPO (Adaptive)} & Clip-aware ESS & \textbf{36.7} & \textbf{68.7} & \textbf{83.8} & \textbf{34.9} & 
        \textbf{46.7} & \textbf{54.2} \\
        \bottomrule
    \end{tabular}
\end{table*}

\section{Computational Complexity Analysis}
\label{app:complexity}

A potential concern regarding the dynamic selection of the exponent $p$ is the additional computational overhead introduced by the iterative bisection solver. In this section, we provide a formal analysis to demonstrate that this overhead is theoretically bounded and practically negligible.
\subsection{Theoretical Complexity}

Let $n$ be the number of response tokens in a trajectory and $K$ be the group size. The computational cost of standard group-based RL methods (e.g., GRPO or GMPO) is primarily $\mathcal{O}(n)$ per trajectoryfor the aggregation step. For \textbf{PMPO}, the complexity for each trajectory is decomposed as: 1) \textbf{Clip Fraction Calculation:} Computing $f_{\mathrm{clip}}$ requires a single pass over the token-level log-probability differences, taking $\mathcal{O}(n)$. 2) \textbf{Bisection Solver:} Each iteration of the bisection method involves computing the Effective Sample Size (ESS) over $n$ tokens. For a fixed number of iterations $M$ (typically $M=10$ to achieve a precision of $10^{-3}$), the total cost is $\mathcal{O}(M \cdot n)$. 3) \textbf{Final Power-Mean Aggregation:} Once $p$ is determined, the effective ratio $\hat{r}_p$ is computed in $\mathcal{O}(n)$.
As a result, the total asymptotic complexity of the PMPO adaptive layer is $\mathcal{O}(M \cdot n)$. Since $M$ is a small constant that does not scale with the model size or the sequence length, PMPO maintains the same linear complexity as its static counterparts.

\subsection{Practical Training Overhead}

We evaluate the practical scalability of PMPO on a server equipped with \textbf{four NVIDIA A800 (80GB)} GPUs. The benchmark uses a Qwen2.5-7B base model with a global batch size of 128 (32 samples per GPU) and a sequence length of 3,000 tokens. The average wall-clock time per training iteration is: 1) \textbf{GRPO} (Arithmetic baseline): 95.6 seconds, \textbf{GMPO} (Geometric baseline): 97.1 seconds, \textbf{PMPO} (Ours): \textbf{97.8 seconds}.
The empirical findings are analyzed as follows:

\begin{enumerate}
    \item \textbf{Marginal Computational Increment:} Compared to the simplest arithmetic mean (GRPO), PMPO incurs a total latency increase of only \textbf{2\%}. When compared to the geometric baseline (GMPO), the additional overhead of our adaptive $p$-solver is a marginal \textbf{0.7\%}. This confirms that the computational bottleneck remains the forward and backward passes of the Transformer backbone, rather than the adaptive aggregation layer.
    \item \textbf{Local Computation and Parallelizability:} Our bisection solver is fully vectorized and operates independently for each trajectory. In a distributed setting, each GPU computes the optimal $p$ for its local micro-batch without requiring inter-GPU communication (e.g., no All-Reduce), ensuring that the overhead does not scale with the number of GPUs.
    \item \textbf{Numerical Stability and Efficiency:} By performing calculations in the log-domain using the Log-Sum-Exp trick, the solver converges reliably within 10 iterations. The measured 3-second difference from GRPO encompasses the cumulative cost of iterative computation, sign-corrected clipping, and kernel launch latencies.
\end{enumerate}

In summary, PMPO provides a sophisticated adaptive mechanism with a computational cost that is orders of magnitude smaller than the core model updates, establishing it as a highly efficient solution for large-scale RLAIF.

\begin{table}[!t]
    \centering
    \small
    \caption{Key training hyperparameters in our implementation. Unless stated otherwise, all compared methods use the same configuration.}
    \label{tab:oat_hparams}
    \begin{tabular}{ll}
        \toprule
        \textbf{Parameter} & \textbf{Value} \\
        \midrule
        \multicolumn{2}{l}{\textbf{Actor / rollout}} \\
        Max response length & 3000 tokens \\
        Sampling temperature & 1.0 \\
        $(top\text{-}p,\;top\text{-}k)$ & $(1.0,\;-1)$ \\
        Responses per question ($K$) & 8 \\
        \midrule
        \multicolumn{2}{l}{\textbf{Learner / optimization}} \\
        Optimizer & AdamW \\
        Adam betas $(\beta_1,\beta_2)$ & $(0.9,\;0.95)$ \\
        Weight decay & 0.0 \\
        Gradient norm clipping & 1.0 \\
        LR scheduler & Constant \\
        Learning rate & $1\times 10^{-6}$ \\
        PPO inner epochs & 1 \\
        KL loss coefficient ($\beta$) & 0.0 \\
        KL penalty coefficient & 0.0 \\
        Policy clipping parameter ($c$) & 0.4 \\
        \bottomrule
    \end{tabular}
\end{table}
\section{Hyper-parameter Setting}
\label{supp:hyper}
We largely align our hyper-parameter configuration with the training recipes of Dr.GRPO~\citep{drgrpo} to ensure a fair comparison.
\textbf{Optimization.} We use the AdamW optimizer with a constant learning rate of $1 \times 10^{-6}$ and gradient norm clipping at $1.0$. Following the DeepSeek-R1-Zero~\citep{deepseek} paradigm, we set the explicit KL penalty coefficient to $0.0$, relying on the policy clipping mechanism ($c=0.4$) for trust-region constraints.
\textbf{Sampling.} For each prompt, we generate $K=8$ trajectories with a sampling temperature of $1.0$ to encourage exploration. The maximum response length is set to $3{,}000$ tokens to accommodate long COT reasoning.
\textbf{PMPO Specifics.} For our adaptive mechanism, we constrain the power-mean exponent $p$ to the range $[0.01, 0.99]$ and set the reliability threshold for ESS matching to $\epsilon_{\mathrm{ess}}=0.1$.


\end{document}